%% file: main.tex
\documentclass[sigconf]{acmart}

\usepackage{hyperref}
\usepackage{url}
\usepackage{booktabs} 
\usepackage{algorithmic}
\usepackage[linesnumbered, boxed,ruled, vlined]{algorithm2e}
\usepackage{multirow}
\usepackage{tablefootnote}
\usepackage{graphicx}
\usepackage{subfigure}
\usepackage{flushend}
\usepackage{color}
\usepackage{enumitem}
\usepackage{array, makecell}
\usepackage{footmisc} 

\usepackage{xcolor}

\newcommand{\yx}[1]{\textbf{\color{red}[(YX: #1 )]}}  
\newcommand{\xz}[1]{\textbf{\color{blue}[(Xu: #1 )]}}

\newcommand{\hide}[1]{} 
\newcommand{\vpara}[1]{\vspace{0.07in}\noindent\textbf{#1 }}
\newcommand{\qinkai}[1]{\textbf{\color{orange}[(Qinkai: #1 )]}}
\newcommand{\name}{TDGIA\xspace}
\newcommand{\beq}[1]{\begin{equation}#1\end{equation}}
\newcommand{\model}{TDGIA\xspace}

\copyrightyear{2021}
\acmYear{2021}
\setcopyright{acmlicensed}\acmConference[KDD '21]{Proceedings of the 27th ACM SIGKDD Conference on Knowledge Discovery and Data Mining}{August 14--18, 2021}{Virtual Event, Singapore}
\acmBooktitle{Proceedings of the 27th ACM SIGKDD Conference on Knowledge Discovery and Data Mining (KDD '21), August 14--18, 2021, Virtual Event, Singapore}
\acmPrice{15.00}
\acmDOI{10.1145/3447548.3467314}
\acmISBN{978-1-4503-8332-5/21/08}

\newtheorem{definition}{Definition}

\begin{document}
\setlength{\marginparwidth}{2cm}
\settopmatter{printacmref=true}
\fancyhead{}

\title{\model: Effective Injection Attacks on Graph Neural Networks}

\author[ X. Zou, Q. Zheng, Y. Dong, X. Guan, E. Kharlamov and J. Tang]{
    Xu Zou$^{\dagger}$, Qinkai Zheng$^{\dagger}$, Yuxiao Dong$^{\ddagger}$, Xinyu Guan$^{\bullet}$
}
\author{Evgeny Kharlamov${^\diamond}$, Jialiang Lu$^{+}$, Jie Tang$^{\dagger\S}$}
\affiliation{
    $^\dagger$ Department of Computer Science and Technology, Tsinghua University,
    $^\ddagger$ Facebook AI
    \country{}
    }
\affiliation{
    $^+$ Shanghai Jiao Tong University,
    $^\bullet$ Biendata,
    $^\diamond$ Bosch Center for Artificial Intelligence
    \country{}
}
\email{
zoux18@mails.tsinghua.edu.cn
}
\email{
  {ericdongyx,guanxinyu}@gmail.com,evgeny.kharlamov@de.bosch.com
 }
 \email{
jialiang.lu@sjtu.edu.cn,
{qinkai,jietang}@tsinghua.edu.cn
}

\thanks{$^\S$Jie Tang is the corresponding author.
}

\renewcommand{\authors}{Xu Zou, Qinkai Zheng, Yuxiao Dong, Xinyu Guan, Evgeny Kharlamov, Jialiang Lu, Jie Tang}

\renewcommand{\shortauthors}{Zou et al}
\begin{abstract}

Graph Neural Networks (GNNs) have achieved promising performance in various real-world applications. However, recent studies find that GNNs are vulnerable to adversarial attacks.  
In this paper, we study a recently-introduced realistic attack scenario on graphs---graph injection attack (GIA). 
In the GIA scenario, the adversary is not able to modify the existing link structure or node attributes of the input graph, instead the attack is performed by injecting adversarial nodes into it. 
We present an analysis on the topological vulnerability of GNNs under GIA setting, based on which we propose the Topological Defective Graph Injection Attack (\model) for effective injection attacks. 
\model first introduces the topological defective edge selection strategy to choose the original nodes for connecting with the injected ones. 
It then designs the smooth feature optimization objective to generate the features for the injected nodes. 
Extensive experiments on large-scale datasets show that \model can consistently and significantly outperform various attack baselines in attacking dozens of defense GNN models. 
Notably, the performance drop on target GNNs resultant from \model is more than double the damage brought by the best attack solution among hundreds of submissions on KDD-CUP 2020.

\end{abstract}

\hide{
As the trend of deep learning expands throughout research areas of computer science and data mining, Graph Neural Networks (GNNs) have achieved promising performances in modeling graph data. Similar to other deep-learning based methods, studies show that GNNs are also vulnerable toward adversarial attacks. In this paper, we study a recent-introduced realistic attack scenario on graphs, Graph Injection Attack (GIA). In the GIA scenario, the adversary is not able to modify the link structure and node attributes of the existing graph, however the attack is performed by injecting adversary nodes. 
We investigate existing methods and present a new GIA method, \model, the method achieves better performance and generalization on three large datasets with hundreds of thousands of nodes and millions of edges. With the help of topological edge selection and smooth feature optimization, \model achieves much better performance than all previous methods and candidate submissions in KDD-CUP 2020 \textit{Graph Adversarial Attack \& Defense Track} on attacking all of the GNN defense models and highly-scored competitor-submitted defense methods in a general manner. Further experiments also demonstrate that \model has pretty good generalization on different datasets over various defense methods. 
}

%
%

\ccsdesc[500]{Security and privacy~Software and application security}
\ccsdesc[500]{Mathematics of computing~Graph algorithms}

\keywords{Graph Neural Networks; Adversarial Machine Learning; Graph Injection Attack; Graph Mining; Graph Adversarial Attack }


\maketitle

\vspace{-0.05in}
\section{Introduction}\label{sec:intro}
\input{1.intro}

\vspace{-0.05in}
\section{Related Works} \label{sec:related}
\input{related}

\vspace{-0.05in}
\section{Problem Definition}\label{sec:problem}
\input{3.problem}

\vspace{-0.05in}

\input{4.model_new}

\vspace{-0.05in}
\section{Experiments}\label{sec:exp}
\input{5.exp_new}

\vspace{-0.05in}
\section{Conclusion} \label{sec:conclusion}
In this work, we explore deeply into the graph injection attack (GIA) problem and present the TDGIA attack method. 
TDGIA consists of two modules: the topological defective edge selection for injecting nodes and smooth adversarial optimization for generating features of injected nodes. 
TDGIA achieves the best attack performance in attacking a variety of defense GNN models, compared with various baseline attack methods including the champion solution of KDD-CUP 2020. 
It is also worth mentioning that with only a few number of injected nodes, TDGIA is able to effectively attack GNNs under the black-box and evasion settings. 

In this work, we mainly focus on leveraging the first level neighborhood on the graph to design the attack strategies. 
In the future, we would like to involve higher levels of neighborhood information for advanced attacks. 
Another interesting finding is that among all the GNN variants, using GCN as the surrogate model achieves the best results. 
Further studies on this observation may deepen our understandings of how different GNN variants work, and thus inspire more effective attack designs. 

We continue our research on graph robustness issues and published Graph Robustness Benchmark on NeurIPS 2021. The work is called \textit{Graph Robustness Benchmark: Benchmarking the Adversarial Robustness of Graph Machine Learning}(GRB) and we also created a website for attack and defense submissions, see GRB.\footnote{\url{https://cogdl.ai/grb/home}}

\clearpage
 \newpage
\bibliographystyle{ACM-Reference-Format}
\bibliography{reference.bib}

\clearpage
\newpage
\appendix
\input{7.appendix}

\end{document}

%% file: 1.intro.tex
Recent years have witnessed widespread adoption of graph machine learning for modeling structured and relational data. 
Particularly, the emergence of graph neural networks (GNNs) has offered promising results in diverse graph applications, such as node classification~\cite{kipf2016semi}, social recommendation~\cite{ying2018graph}, and drug design~\cite{jiang2020drug}. 

Despite the exciting progress, studies have shown that neural networks are commonly vulnerable to adversarial attacks, where 
slight, imperceptible but intentionally-designed perturbations on inputs can cause incorrect predictions~\cite{szegedy2013intriguing, goodfellow2014generative, huang2017adversarial}. 
Attacks on general neural networks usually focus on modifying the attributes/features of the input instances, such as minor perturbations in individual pixels of an image. 
Uniquely, adversarial attacks can also be applied to graph-structured input, requiring dedicated strategies for exploring the specific vulnerabilities of the underlying models.

The early attacks on GNNs usually follow the setting of graph modification attack (GMA)~\cite{dai2018adversarial,zugner2018adversarial,sun2020adversarial, zugner2019adversarial}, as illustrated in Figure \ref{fig:tdgia} (a): given an input graph with attributes, the adversary can directly modify the links between its nodes (red links) and the attributes of existing nodes (red nodes). 
However, in real-world scenarios, it is often unrealistic for the adversary to get the authority to modify existing data. 
Take the citation graph for example, it automatically forms when papers are published, making it practically difficult to change the citations and attributes of one publication afterwards. 
But what is relatively easy is to inject new nodes and links into the existing citation graph, e.g., by ``publishing'' fake papers, to mislead the predictions of GNNs. 


In view of the gap, very recent efforts~\cite{sun2020adversarial, wang2020scalable}, including the KDD-CUP 2020 competition\footnote{
\url{https://www.biendata.xyz/competition/kddcup_2020_formal/}}, have been devoted to adversarial attacks on GNNs under the setting of graph injection attack (GIA). 
Specifically, the GIA task in KDD-CUP 2020 is formulated as follows: 
(1) Black-box attack, where the adversaries do not have access to the target GNN model or the correct labels of the target nodes; 
(2) Evasion attack, where the attacks can only be performed during the inference stage. 
The GIA settings present unique challenges that are not faced by GMA, such as how to connect existing nodes with injected nodes and how to generate features for injected nodes from scratch. 
Consequently, though numerous attacks are submitted by hundreds of teams, the resultant performance drops are relatively limited and no principled models emerge from them.

\begin{figure*}
    \centering
    \mbox
    {
    \begin{subfigure}[GMA vs. GIA]{
        \centering
         \includegraphics[height=0.51\columnwidth]{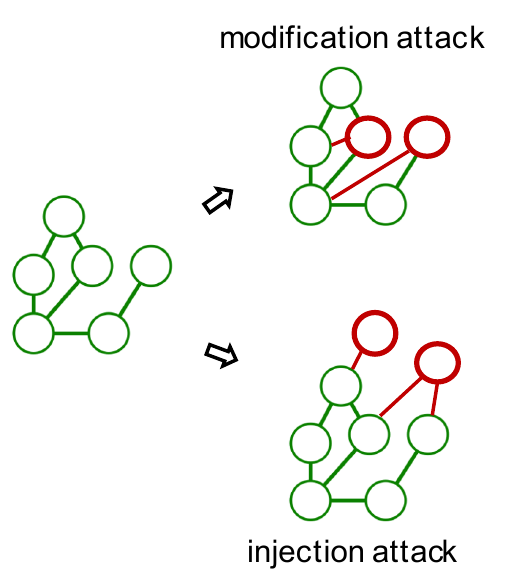}   
         }
    \end{subfigure}
    
    \hspace{0.15in}
    \begin{subfigure}[Topological Defective Graph Injection Attack (\model)]{
        \centering
         \includegraphics[height=0.53\columnwidth]{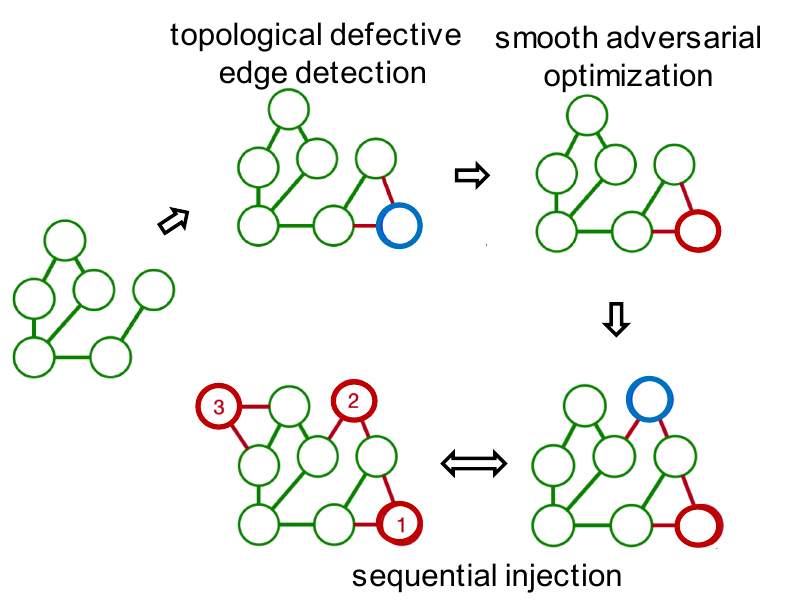}   
         }
    \end{subfigure}

    \begin{subfigure}[TDGIA vs. The Best Results in KDD-CUP 2020]{
        \centering
         \includegraphics[height=0.52\columnwidth]{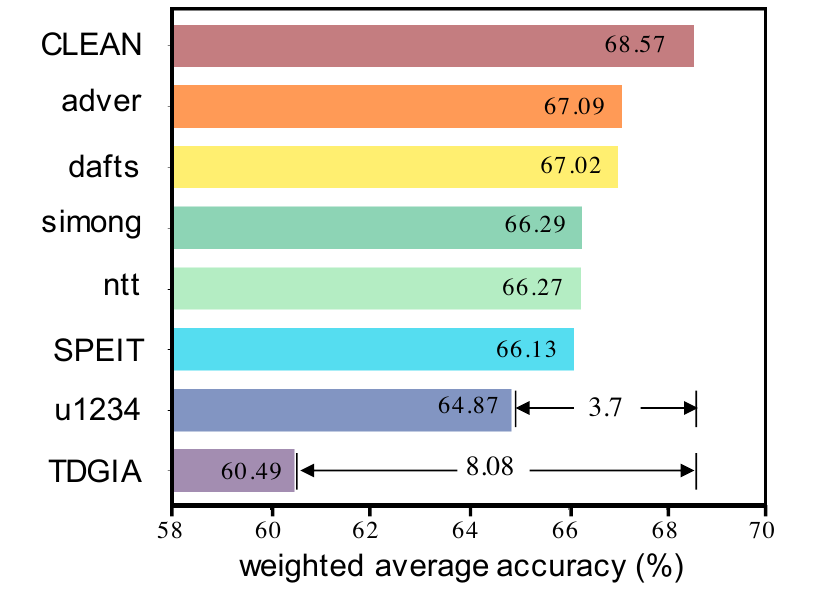}   
         }
    \end{subfigure}
    }
    \caption{An overview of graph injection attack, the proposed \model method, and its performance. 
    }
   \label{fig:tdgia}
\end{figure*}

\vpara{Contributions.}In this work, we study the problem of GIA under the black-box and evasion attack settings, where the goal is to design an effective injection attack framework that can best fool the target GNN models and thus worsen their prediction capability. 
To understand the problem and its challenges, we present an in-depth analysis of the vulnerability of GNNs under the graph injection attack and show that GNNs, as non-structural-ignorant models, are GIA-attackable. 
Based on this, we present the topological defective graph injection attack (\model) (Cf. Figure \ref{fig:tdgia} (b)). 
\model consists of two modules---topological defective edge selection and smooth adversarial optimization for injected node attribute generation---that corresponds to the problem setup of GIA. 
Specifically, we leverage the topological vulnerabilities of the original graph to detect existing nodes that can best help the attack and then inject new nodes surrounding them in a sequential manner. 
With that, we design a smooth loss function to optimize the nodes' features for minimizing the the performance of the target GNN model.

Both the studied problem and the proposed \model method differ from existing (injection) attack methods. 
Table \ref{tab:diffs} summarizes the differences. 
First, NIPA~\cite{sun2020adversarial} and AFGSM~\cite{wang2020scalable} are developed under the poison setting, which requires the re-training of the defense models for each attack. 
Differently, \model follows KDD-CUP 2020 to use the evasion attack setting, where different attacks are evaluated based on the same set of models and weights. 
Second, the design of \model enables it to attack large-scale graphs that can not be handled by the reinforcement-learning based NIPA. 
Third, compared with AFGSM, \model proposes a more general way to consume the topological information, resulting in significant performance improvements. 
Finally, attacks in previous works are only evaluated on weak defense models like raw GCN, while \model is shown to be effective even against the top defense solutions examined in KDD-CUP 2020.  

 
We conduct extensive experiments on large-scale datasets to demonstrate the performance and transferability of the proposed attack method.  
Figure \ref{fig:tdgia} (c) lists the results for \model and the top submissions on KDD-CUP 2020 as measured by weighted average accuracy. 
The experimental results show that \model significantly and consistently outperforms various baseline methods. 
For example, the best KDD-CUP 2020 attack ({u1234}) can make the performance of the target GNN models drop $3.7\%$, while \model can drag its performance down by $8.08\%$---a $118\%$ increase in damage. 
Moreover, \model achieves this attack performance by injecting only a limited number of nodes ($1\%$ of target nodes). 
Additionally, various ablation studies demonstrate the effectiveness of each module in \model. 


To sum up, this work makes the following contributions: 
\begin{itemize}
	\item We study the GIA problem with the black-box and evasion settings, and theoretically show that non-structural-ignorant GNN models are vulnerable to GIA. 
	
	\item We develop the Topological Defective Graph Injection Attack (\model) that can explore and leverage the vulnerability of GNNs and the topological properties of the graph. 

	\item We conduct experiments that consistently demonstrate \model's significant outperformance over baselines (including the best attack submission at KDD-CUP 2020) against various defense GNN models across different datasets. 
\end{itemize}


\hide{
The rest of this paper is organized as follows: In Section \ref{sec:related}, we introduce related work on graph adversarial attacks. In Section \ref{sec:problem}, we formulate the problem of GMA and GIA respectively, and give a theoretical proof of the vulnerability of GNNs against GIA. In Section \ref{sec:model}, we present proposed attack method GIA. In Section \ref{sec:exp}, we show extensive experiments to show the effectiveness of our method. 
}
\begin{table}

\small
\caption{Summary of graph adversarial attacks.}	
\label{tab:diffs}
\begin{tabular}{|c|c|c|c|c|}
\hline \hline
\textbf{Attack} & \textbf{Task}  & \textbf{Type} & \textbf{Method} & \begin{tabular}[c]{@{}c@{}} \textbf{Attack} \\ \textbf{Setting} \end{tabular} \\ \hline
\citeauthor{dai2018adversarial}~\cite{dai2018adversarial} & \begin{tabular}[c]{@{}c@{}}Node cl.\\ Graph cl.\end{tabular} & GMA  & \begin{tabular}[c]{@{}c@{}} Reinforcement \\ learning \end{tabular}  & Evasion                                                            \\ \hline
Nettack~\cite{zugner2018adversarial} & \begin{tabular}[c]{@{}c@{}}Node cl.\\ Graph cl.\end{tabular} & GMA & \begin{tabular}[c]{@{}c@{}}Greedy algorithm \\ Model linearization\end{tabular} & \begin{tabular}[c]{@{}c@{}} Evasion \\ Poison \end{tabular}  \\ \hline
Meta~\cite{zugner2019adversarial} & Node cl. & GMA & Meta-learning & Poison \\ \hline
NIPA~\cite{sun2020adversarial}  & Node cl. & GIA & \begin{tabular}[c]{@{}c@{}} Reinforcement \\ learning \end{tabular} & Poison \\ \hline
AFGSM~\cite{wang2020scalable} & Node cl. & GIA & \begin{tabular}[c]{@{}c@{}}Fast Gradient Sign \\ Model linearization\end{tabular} & Poison \\ \hline
\name (ours) & Node cl. & GIA & \begin{tabular}[c]{@{}c@{}} Defective edge selection \\ Smooth optimization \end{tabular} & Evasion \\ \hline \hline                                                         
\end{tabular}
\end{table}

%% file: related.tex
\vpara{Adversarial Attacks on Neural Networks.}
The phenomenon of adversarial examples against deep-learning based models is first discovered in computer vision~\cite{szegedy2013intriguing}. Adding delicate and imperceptible perturbations to images can significantly change the predictions of deep neural networks. ~\cite{goodfellow2014explaining} proposes Fast Gradient Sign Method (FGSM) to generate this kind of perturbations by using the gradients of trained neural networks. Since then, more advanced attack methods are proposed~\cite{carlini2017towards, athalye2018obfuscated}. Adversarial attacks also grow in a wider range of fields such as natural language processing~\cite{li2019textbugger}, or speech recognition~\cite{carlini2018audio}. Nowadays, adversarial attacks have become one of the major threats to neural networks.

\vpara{Adversarial Attacks on GNNs.}
Due to the existence of adversarial examples, the vulnerability of graph learning algorithms has also been revealed. By modifying features and edges on graph-structured data, the adversary can significantly degrade the performance of GNN models. 
As shown in Table \ref{tab:diffs}, ~\cite{dai2018adversarial} proposes a reinforcement-learning based attack on both node classification and graph classification tasks. This attack only modifies the structure of graph. 
~\cite{zugner2018adversarial} proposes Nettack, the first adversarial attack on attributed graphs. It shows that only few perturbations on both edges and features of the graph can be extremely harmful for models like GCN. Nettack uses a greedy approximation scheme under the constraints of unnoticeable perturbations and incorporates fast computation.
Furthermore, ~\cite{zugner2019adversarial} proposes a poison attack on GNNs via meta-learning, Meta-attack, which only modifies a small part of the graph but can still decrease the performance of GNNs on node classification tasks remarkably. A recent summary \cite{jin2021adversarial} summarizes different graph adversarial attacks.

A more realistic scenario, graph injection attack (GIA), is studied in~\cite{sun2020adversarial, wang2020scalable}, which injects new vicious nodes instead of modifying the original graph. ~\cite{sun2020adversarial} proposes Node Injection Poisoning Attack (NIPA) based on reinforcement learning strategy. 
Under the same scenario, Approximate Fast Gradient Sign Method (AFGSM)~\cite{wang2020scalable} further uses an approximation strategy to linearize the model and to generate the perturbations efficiently. These works are under the \textit{poison setting}, where models have to be re-trained after the vicious nodes are injected onto the graph.

\vpara{KDD-CUP 2020 Graph Adversarial Attack \& Defense.}
KDD-CUP 2020 introduces the GIA scenario in \textit{Graph Adversarial Attack \& Defense} competition track. The adversary doesn't have access to the target model (i.e. \textit{black-box setting}), and can only inject no more than 500 nodes to a large citation network with more than 600,000 nodes and millions of links.  

For attackers, the attack is conducted during the inference stage, a.k.a. the \textit{evasion setting}. For defenders, they should provide robust GNN models to resist various adversarial attacks. Hundreds of attack and defense solutions are submitted during the competition, which serve as good references for further development of graph adversarial attacks.

 Another similar setting is introduced in KDD-CUP 2020. The Graph Adversarial Attack \& Defense Track introduced a setting that the injection nodes are injected only during the inference stage. To separate this setting from the \textit{poison attack} setting, we call this \textit{evasion attack}.

 \begin{table*}[!htpb]
 \label{tab:attacks}
 \caption{Summary of several adversarial attacks on Graph Neural Networks.}
 	
 \begin{tabular}{ccccc}
 \hline \hline
 \textbf{Attacks} & \textbf{Task}                                                                       & \textbf{Modification} & \textbf{Methodology}                                                                     & \textbf{Attack Stage}                                                      \\ \hline
 \citeauthor{dai2018adversarial}~\cite{dai2018adversarial}                         & \begin{tabular}[c]{@{}c@{}}Node classification,\\ Graph classification\end{tabular} & Edges                 & Reinforcement learning                                                                   & Evasion attack                                                             \\ \hline
 Nettack~\cite{zugner2018adversarial}                      & \begin{tabular}[c]{@{}c@{}}Node classification.\\ Graph classification\end{tabular} & Edges \& Features     & \begin{tabular}[c]{@{}c@{}}Greedy algorithm,\\ Model linearization\end{tabular}          & \begin{tabular}[c]{@{}c@{}}Evasion attack,\\ Poisoning attack\end{tabular} \\ \hline
 Meta-attack~\cite{zugner2019adversarial}                  & Node classification                                                                 & Edges \& Features     & Meta-learning                                                                            & Poison attack                                                           \\ \hline
 NIPA~\cite{sun2020adversarial}                         & Node classification                                                                 & Node injection        & Reinforcement learning                                                                   & Poison attack                                                           \\ \hline
 AFGSM~\cite{wang2020scalable}                        & Node classification                                                                 & Node injection        & \begin{tabular}[c]{@{}c@{}}Fast Gradient Sign Method,\\ Model linearization\end{tabular} & Poison attack                                                           \\ \hline
 \name (ours)                   & Node classification                                                                 & Node injection        & \begin{tabular}[c]{@{}c@{}} Defective edge selection,\\ Feature optimization \end{tabular}                                                                                         & Evasion attack \\ \hline \hline                                                         
 \end{tabular}
\end{table*}

%% file: 3.problem.tex

Broadly, there are two types of graph adversarial attacks: graph modification attack (GMA) and graph injection attack (GIA). 
The focus of this work is on graph injection attack. 
We formalize the attack problem and introduce the attack settings.

The problem of graph adversarial attack was first formalized in Nettack ~\cite{zugner2018adversarial}, which we name as graph modification attack. 
Specifically, we define an attributed graph $\mathbf{G}=(\mathbf{A},\mathbf{F})$ with $\mathbf{A}\in{\mathbb{R}^{N\times N}}$ being the adjacency matrix of its $N$ nodes and $\mathbf{F}\in{\mathbb{R}^{N\times D}}$ as its $D$-dimensional node features. 
Let $\mathcal{M}:\mathbf{G}\rightarrow \{1,2,...,C\}^N$ 
be a model that predicts the labels for all $N$ nodes in $\mathbf{G}$, and denote its predictions as $\mathcal{M}(\mathbf{G})$. 
The goal of GMA is to minimize the number of correct predictions of $\mathcal{M}$ on a set of target nodes $\mathcal{T}$ by modifying the original graph $\mathbf{G}$:
\beq{
\label{eq:gma}
\begin{split}
& \min_{\mathbf{G}'}|\{\mathcal{M}(\mathbf{G}')_i=y_i, i\in\mathcal{T}\}| \\
s.t.\ & \mathbf{G}'=(\mathbf{A}', \ \mathbf{F}'),f_{\Delta_{\mathbf{A}}}(\mathbf{A}'-\mathbf{A})+f_{\Delta_{\mathbf{F}}}(\mathbf{F}'-\mathbf{F})\leq\Delta
\end{split}
}

\noindent where 
$\mathbf{G}'$ is the modified graph, 
$y_i$ is the ground truth label of node $i$, 
$f_{\Delta_{\mathbf{A}}}$ and $f_{\Delta_{\mathbf{F}}}$ are pre-defined functions that measure the scale of modification.
The constraint $\Delta$ ensures that the graph can only be slightly modified by the adversary.

\vpara{Graph Injection Attack.}
Instead of GMA's modifications of $\mathbf{G}$'s structure and attributes, GIA directly injects $N_I$ new nodes into $\mathbf{G}$ while keeping the original edges and attributes of $N$ nodes unchanged. Formally, GIA constructs $\mathbf{G}'=(\mathbf{A}',\mathbf{F}')$ with
\beq{
\label{eqn:A'}
	\mathbf{A}' = 
	\left[ \begin{array}{cc} 
	\mathbf{A} & \mathbf{V}_I \\ 
	\mathbf{V}_I^T & \mathbf{A}_I \\
    \end{array} \right], 
    \mathbf{A}\in{\mathbb{R}^{N\times N}},
    \mathbf{V}_I\in{\mathbb{R}^{N\times N_I}},
    \mathbf{A}_I\in{\mathbb{R}^{N_I\times N_I}}
}
\beq{
\label{eqn:F'}
	\mathbf{F}' = 
    \left[ \begin{array}{cc} 
	\mathbf{F}  \\ 
	\mathbf{F}_I \\
    \end{array} \right],
    \mathbf{F}\in{\mathbb{R}^{N\times D}},
    \mathbf{F}_I\in{\mathbb{R}^{N_I\times D}}
}

\noindent where 
$\mathbf{A}_I$ is the adjacency matrix of the injected nodes, 
$\mathbf{V}_I$ is a matrix that represents edges between $\mathbf{G}$'s original nodes and the injected nodes, 
and $\mathbf{F}_I$ is the feature matrix of the injected nodes. 
The objective of GIA can be then formalized as:
\beq{
\label{eq:2}
\begin{split}
& \min_{\mathbf{G}'} |\{\mathcal{M}(\mathbf{G}')_i=y_i, i\in\mathcal{T}\}| \\
s.t.\ & \mathbf{G}'=(\mathbf{A}',\mathbf{F}'), N_I \leq b, deg(v)_{v\in I}\leq d, ||\mathbf{F}_I||\leq \Delta_F
\end{split}
}
\noindent where $I$ is the set of injected nodes, $N_I$ is limited by a budget $b$, each injected node's degree is limited by a budget $d$, 
and the norm of injected features are restricted by $\Delta_F$. 
These constraints are to ensure that GIA is as unnoticeable as possible by the defender.

\vpara{The Attack Settings of GIA.}
GIA has recently attracted significant attentions and served as one of the KDD-CUP 2020 competition tasks. 
Considering its widespread significance in real-world scenarios, we follow the same settings used in the competition, that is, black-box and evasion attacks. 

\textit{Black-box attack.}
In the black-box setting, the adversary does not have access to the target model $\mathcal{M}$, including its architecture, parameters, and defense mechanism. 
However, the adversary is allowed to access the original attributed graph $\mathbf{G}=(\mathbf{A},\mathbf{F})$ and labels of training and validation nodes but not the ones to be attacked. 

\textit{Evasion attack.}
Straightforwardly, GIA follows the evasion attack setting in which the attack is only performed to the target model during inference. 
This makes it different from the poison attack ~\cite{sun2020adversarial, wang2020scalable}, where the target model is retrained on the attacked graph. 

In addition, the scale of the KDD-CUP 2020 dataset is significantly larger than those commonly used in existing graph attack studies~\cite{zugner2018adversarial, zugner2019adversarial, wang2020scalable}. 
This makes the task more relevant to real-world applications and also requires more scalable  attacks. 

\vpara{The GIA Process.}
\label{sec:GIA}
Due to the black-box and evasion settings, GIA needs to conduct transfer attack with the help of surrogate models as done in~\cite{dai2018adversarial, zugner2018adversarial}.
First, a surrogate model is trained; 
Second, injection attack is performed on this model; 
Finally, the attack is transferred to one or more target models.

To handle large-scale datasets, GIA can be separated into two steps based on its definition. 
First, the edges between existing nodes and injected nodes ($\mathbf{A_I}$, $\mathbf{V}$) are generated; 
Second, the features of injected nodes are optimized. 
This breakdown can largely reduce the complexity and make GIA applicable to large-scale graphs.

The target model $\mathcal{M}$ could be any graph machine learning models. 
Following the community convention~\cite{zugner2018adversarial, zugner2019adversarial, wang2020scalable}, the focus of this work is on graph neural networks as the target models. 


%% file: 4.model_new.tex

\section{\hspace{-0.2cm}GNNs under Graph Injection Attack}
\label{sec:gnn}
In this section, we analyze the behavior of graph neural networks (GNNs) under the general injection attack framework. 
The analysis results can be used to design effective GIA strategies.

\subsection{The Vulnerability of GNNs under GIA}

Intuitively, the function of GIA requires the injected nodes to spread (misleading) information over edges in order to influence other (existing) nodes. Which kind of models are vulnerable to such influence?
\hide{If a graph ML model ignores the edges generated by GIA, the attack won't have any effect. 
For example, the multi-layer perceptron (MLP) model, which only classifies nodes according to their features, can not be affected by GIA (evasion). }
In this section, we investigate the vulnerability of GNN models under GIA.

\begin{definition}[Permutation Invariant]
Given $\mathbf{G}=(\mathbf{A}, \mathbf{F})$, the graph ML model $\mathcal{M}$ is permutation invariant, if for any  $\mathbf{G}'=(\mathbf{A}',\mathbf{F}')$ with $\mathbf{G}'$ as a permutation of $\mathbf{G}$ such that $\forall i\in\{1,2,...,N\}, \mathcal{M}(\mathbf{G}')_{\sigma_i}=\mathcal{M}(\mathbf{G})_i$. 
Note that $\mathbf{G}'$ is a permutation of $\mathbf{G}$, if there exists a permutation $\sigma$: $\{1, ..., N\} \rightarrow \{\sigma_1, ..., \sigma_N\}$ such that $\forall i\in\{1,2,...,N\}, \mathbf{F}'_{\sigma_i}=\mathbf{F}_i$ and $\forall (i,j)\in\{1,2,...,N\}^2, \mathbf{A}'_{\sigma_i\sigma_j}=\mathbf{A}_{ij}$. 
\end{definition}

\begin{definition}[Gia-Attackable]
\label{def:gia-attackable}
The model $\mathcal{M}$ is GIA-attackable, if there exist two graphs $\mathbf{G_1}$ and $\mathbf{G_2}$ containing the same node $i$, such that $\mathbf{G_1}$ is an induced subgraph of $\mathbf{G_2}$ and $\mathcal{M}(\mathbf{G_1})_i\neq \mathcal{M}(\mathbf{G_2})_i$.
\end{definition}

A GIA-attackable graph ML model is a model that an attacker can change its prediction of a certain node by injecting nodes into the original graph. 
By definition, the index of node $i$ in $\mathbf{G_1}$ and $\mathbf{G_2}$ do not matter for \textit{permutation invariant} models, since permutations can be applied to make it to index 0 in both graphs.

\begin{definition}[Structural-Ignorant Model]
\label{def:3}
The model $\mathcal{M}$ is a structural-ignorant model, if $\forall \mathbf{G}_1=(\mathbf{A}_1,\mathbf{F}),\mathbf{G}_2=(\mathbf{A}_2,\mathbf{F})$ such that $\forall i\in\{1,2,...,N\}$, $\mathcal{M}(\mathbf{G}_1)_i=\mathcal{M}(\mathbf{G}_2)_i$, that is, $\mathcal{M}$ gives the same predictions for nodes that have the same features $\mathbf{F}$. 
On the contrary, $\mathcal{M}$ is non-structural-ignorant, if there exist $\mathbf{G}_1=(\mathbf{A}_1,\mathbf{F}),\mathbf{G}_2=(\mathbf{A}_2,\mathbf{F})$,  $\mathbf{A}_1\neq\mathbf{A}_2$, and $\mathcal{M}(\mathbf{G}_1)_i\neq\mathcal{M}(\mathbf{G}_2)_i$.
\end{definition}

According to this definition, most GNNs are non-structural-ignorant, as they rely on the graph structure $\mathbf{A}$ for node classification instead of only using node features $\mathbf{F}$. 
 We use a lemma to show that non-structural-ignorant models are GIA-attackable, the proof of the lemma is included in the Appendix \ref{app:proof}. According to the lemma, we demonstrate that if a permutation-invariant graph ML model is not structural-ignorant, it is GIA-attackable.

\begin{lemma}[Non-structural-ignorant Models are GIA-Attackable]
\label{lemma:GIA}
If a model $\mathcal{M}$ is non-structural-ignorant and permutation invariant, $\mathcal{M}$ is GIA-attackable.
\end{lemma}
\subsection{Topological Vulnerability of GNN Layers}
\label{sec:gnn-topo}

\begin{figure}
    
   \centering
    \includegraphics[width=0.45\textwidth]{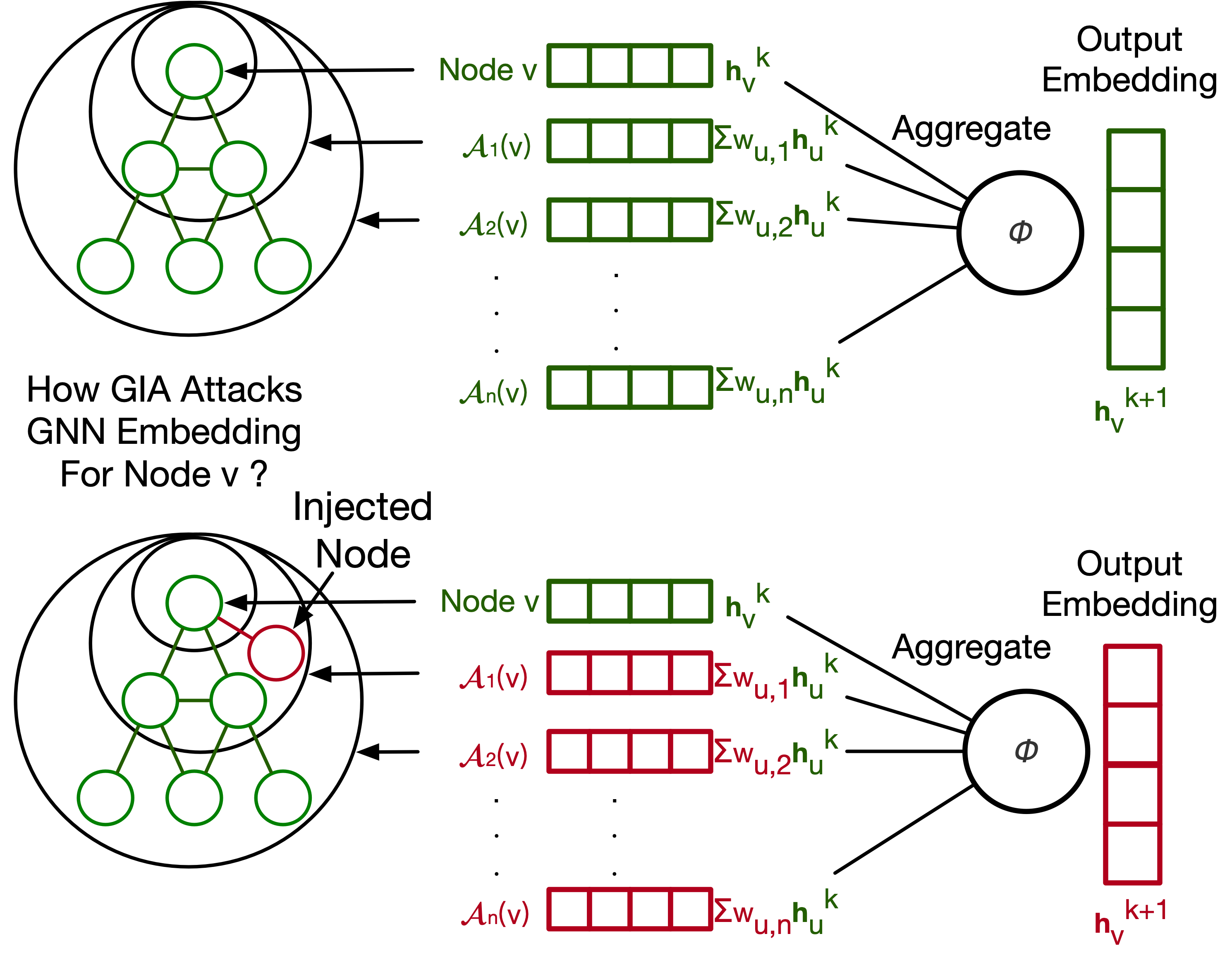}
    \hspace{-0.1in}
    \caption{Top: A GNN layer aggregates information from $n$-hop neighbors of node $v$. Bottom: GIA perturbs the output embedding through $1$-hop node injection.}
    
    \label{fig:general}
\end{figure}

In order to better design injection attacks to GNNs, we explore the topological vulnerabilities of GNNs. 
Generally, a GNN layer performed on a node $v$ can be represented as the aggregation process~\cite{hamilton2017inductive,xu2018powerful}:
\beq{
\label{eqn:agg}
\mathbf{h}_{v}^{k}=\phi(\mathbf{h}_v^{k-1},\mathbf{f}(\{\mathbf{h}_u^{k-1}\}_{u\in \mathcal{A}(v)}))
}
where 
$\phi(\cdot)$ and $\mathbf{f}(\cdot)$ are vector-valued functions, 
$\mathcal{A}(v)$ denotes the neighborhood of node $v$, 
and $\mathbf{h}_{v}^{k}$ is the vector-formed hidden representation of node $v$ at layer $k$. 


Note that $\mathcal{A}(v)$ include nodes that are directly connected to $v$ and nodes that can be connected to $v$ within certain number of steps. 
We use $\mathcal{A}_{t}(v)$ to represent the $t$-hop neighbors of $v$, i.e. nodes that can reach $v$ within $t$ steps. We use $\mathbf{f}_t(\cdot)$ as the corresponding aggregation functions. 
Therefore, Eq. \ref{eqn:agg} can be further expressed by
\beq{
\begin{split}
\mathbf{h}_{v}^{k}=&\phi(\mathbf{f}_0(\mathbf{h}_v^{k-1}),\mathbf{f}_1(\{\mathbf{h}_u^{k-1}\}_{u\in \mathcal{A}_{1}(v)}),\mathbf{f}_2(\{\mathbf{h}_u^{k-1}\}_{u\in \mathcal{A}_{2}(v)}),\\&...,\mathbf{f}_n(\{\mathbf{h}_u^{k-1}\}_{u\in \mathcal{A}_{n}(v)}))
\end{split}
}
Suppose we perturb the graph by injecting nodes so that $\mathcal{A}_t$ changes to $\mathcal{A}'_t$, and $\mathbf{f}_t$ is changed by a comparable small amount, i.e.,  $\mathbf{f}_t'-\mathbf{f}_t = \Delta \mathbf{f}_t$. The new embedding for $v$ at layer $k$ becomes  
\beq{
\begin{split}
\mathbf{h}_{v}'^k= &\mathbf{h}_{v}^k+\frac{\partial\phi}{\partial \mathbf{f}_0}(\mathbf{f}_0'-\mathbf{f}_0)+\frac{\partial \phi}{\partial \mathbf{f}_1}(\mathbf{f}_1'-\mathbf{f}_1)+...+\frac{\partial \phi}{\partial \mathbf{f}_n}(\mathbf{f}_n'-\mathbf{f}_n)\\&+O((\Delta \mathbf{f}_0)^2+(\Delta \mathbf{f}_1)^2+...+(\Delta \mathbf{f}_n)^2)
\end{split}}
By denoting $\frac{\partial \phi}{\partial \mathbf{f}_t}$ at layer $k$ as $\mathbf{p}_{t,k}$, we have 
\beq{
\label{eq:atk}
\Delta \mathbf{h}_{v}^k=\mathbf{h}_{v}'^k-\mathbf{h}_v^k=\sum_{t=0}^{n}\mathbf{p}_{t,k}\Delta \mathbf{f}_t+o(\sum_t\Delta \mathbf{f}_t)
}
Without loss of generality, we assume that $\mathbf{f}_t$ has the form of weighted average that is widely used in GNNs~\cite{kipf2016semi, wu2019simplifying, zhu2019robust}
\beq{
\mathbf{f}_t(\{\mathbf{h}_u^{k-1}\}_{u\in \mathcal{A}_t(v)})=\sum_{u\in \mathcal{A}_t(v)} w_{u,t}\mathbf{h}_u^{k-1}
}
where the weight $w_{u,t}$ corresponds to the topological structure within $t$-hop neighborhood of node $v$. 
Therefore, we can exploit the topological vulnerabilities of GNNs under GIA setting by conducting $t$-hop node injection to construct $\mathcal{A}_t'(v)$ and to perturb the output of $\mathbf{f}_t$
\beq{
\label{eq:uu}
\Delta \mathbf{f}_t=\sum_{u\in \mathcal{A}'_t(v)}(\Delta w_{u,t}\mathbf{h}_{u}^{k-1}+w_{u,t}\Delta \mathbf{h}_{u}^{k-1})
}

Figure \ref{fig:general} represents an example of 1-hop node injection. The injected node can affect the embeddings of the $t$-hop neighborhoods of node $v$, resulting in final misclassification after the aggregation. The problem now becomes how to harness the topological vulnerabilities of the graph to conduct effective node injection. 

\section{The \model Framework}
\label{sec:model}

We present the Topological Defective Graph Injection Attack (\model) framework for effective attacks on graph neural networks. 
Its design is based on the vulnerability analysis of GNNs in Section \ref{sec:gnn}. 
The overall process of \model is illustrated in Figure~\ref{fig:tdgia} (b). 

Specifically, \model consists of two steps that corresponds to the general GIA process: topological defective edge selection and smooth adversarial optimization. 
First, we identify important nodes according to the topological properties of the original graph and inject new nodes around them sequentially. 
Second, to minimize the performance of node classification, we optimize the features of the injected nodes with a smooth loss function. 

\subsection{Topological Defective Edge Selection}
\label{sec:defective}

In light of the topological vulnerability of a GNN layer (Cf. Section \ref{sec:gnn-topo}), we design an edge selection scheme to generate defective edges between injected nodes and original nodes to attack GNNs. 

For a node $v$, \model can change its $t$-hop neighbors $\mathcal{A}_t(v)$ to $\mathcal{A}'_t(v)$. For original nodes $u\in\mathcal{A}_t(v)$, their features $\mathbf{h}_u$ remain unchanged. For injected nodes $u\in\mathcal{A}'_t(v)\backslash\mathcal{A}_t(v)$, their features $\textbf{h}_u$ are initialized by zeros, then $\Delta \textbf{h}_u=\textbf{h}'_u$. Then, Eq. \ref{eq:uu} can be developed into
\beq{
\label{eqn:perturb}
\Delta \mathbf{f}_t=\sum_{u\in \mathcal{A}_t(v)}\Delta w_{u,t}{\mathbf{h}_{u}^{k-1}}+\sum_{u\in \mathcal{A}'_t(v)\backslash\mathcal{A}_t(v)}w_{u,t} {\mathbf{h}_u'^{k-1}}
}

We start from attacking a single-layer GNN, i.e. $k=1$.
According to Eq. \ref{eq:atk}, we shall maximize $\sum_{t=0}^{1}\mathbf{p}_{t1}\Delta \mathbf{f}_t$ to maximize $\Delta \mathbf{h}_{v}^1$. 
Note that $\Delta\mathbf{f}_0=\Delta \textbf{h}_v^0=\mathbf{0}$ since the original features are unchanged. 
Thus, we only need to maximize $\Delta \mathbf{f}_1$
\beq{
\label{eqn:perturb2}
\Delta \mathbf{f}_1=\sum_{u\in \mathcal{A}_1(v)}\Delta w_{u,1}{\mathbf{h}_u^0}+\sum_{u\in \mathcal{A}'_1(v)\backslash\mathcal{A}_1(v)}w_{u,1} {{\mathbf{h}}_u'^0}
}

During edge selection stage, the features of injected nodes are not yet determined. Thus, our strategy is to first maximize the influence on $\sum w_{u,1}, u\in \mathcal{A}'_t(v)$ and to perturb $\Delta \mathbf{f}_1$ as much as possible. 

We start from the common choices of $w_{u,1}$ used in GNNs. 
Following GCN~\cite{kipf2016semi}, various types of GNNs~\cite{wu2019simplifying, zhu2019robust, du2017topology} use
\beq{
\label{eqn:w1}
w_{u,1}=\frac{1}{\sqrt{deg(u)deg(v)}}, u\in \mathcal{A}_1(v)
}
while mean-pooling based GNNs like GraphSAGE~\cite{hamilton2017inductive} use 
\beq{
\label{eqn:w2}
w_{u,1}=\frac{1}{deg(v)}, u\in \mathcal{A}_1(v)
}

In \model, when deciding which nodes the injected nodes should be linked to, we use a combination of weights from Eq. \ref{eqn:w1} and Eq. \ref{eqn:w2} to scale the topological vulnerability of node $v$:
\beq{
\label{eqn:w_tdgia}
\lambda_v=k_1\frac{1}{\sqrt{deg(v)d}}+k_2\frac{1}{deg(v)}
}
where $deg(v)$ is the degree of the target node $v$ and $d$ is the budget on degree of injected nodes. The higher $\lambda_v$ is, the more likely a node may be attacked by GIA. In \model, we connect injected nodes to existing nodes with higher $\lambda_v$ by constructing defective edges.

In Appendix~\ref{app:layer}, we theoretically demonstrate that Eq. \ref{eq:uu} can be generalized to multi-layer GNNs, thus this topological defective edge selection strategy still works under general GNNs. 

\subsection{Smooth Adversarial Optimization}
\label{sec:smooth}
Once the topological defective edges of injected nodes have been selected, the next step is to generate features for the injected nodes to advance the effect of the attacks. 
Specifically, given a model $\mathcal{M}$, an adjacency matrix $\mathbf{A}'$ after node injection, we further optimize the features $\mathbf{F}_{I}$ of the injected nodes in order to (negatively) influence the model prediction $\mathcal{M}(\mathbf{A}', \mathbf{F}^{'})$. 
To that end, we design a smooth adversarial feature optimization with a smooth loss function. 

\vpara{Smooth Loss Function.}
Usually, in adversarial attack, we optimize reversely the loss function used for training a model. For example, we can use the inverse of KL divergence as the attack loss for a node $v$ in target set $\mathcal{T}$ 
\beq{
\label{eqn:loss}
\mathcal{L}_v=-D_{KL}(\mathbf{Y}_{\text{pred}}||\mathbf{Y}_{\text{test}})=\ln(p_{y_{v,\text{pred}}=y_{v,\text{test}}})=\ln p_v
}

\noindent where $p_v$ is for simplicity the probability that $\mathcal{M}$ correctly classifies $v$. 
Using this loss may cause gradient explosion, as the derivative 
\beq{
\label{eqn:diff_1}
\frac{\partial \mathcal{L}_v}{\partial p_v}=\frac{\partial \ln p_v}{\partial p_v}=\frac{1}{p_v}
}

\noindent goes to $\infty$ when $p_v\rightarrow 0$. To prevent such unstable behavior during optimization, we use a smooth loss function 
\beq{
\label{eqn:loss_tdgia}
\mathcal{L}_v=\max(r+\ln p_v,0)^2
}

\noindent where $r$ is a control factor. Therefore the derivative becomes
\beq{
\frac{\partial \mathcal{L}_v}{\partial p_v}=
\begin{cases}
\frac{2(r+\ln p_v)}{p_v}, & e^{-r} < p_v \leq 1 \\ 
0, & 0 \leq p_v \leq e^{-r}
\end{cases}}

\noindent where $\frac{\partial \mathcal{L}_v}{\partial p_v}\rightarrow 0$ when $p_v\rightarrow 0$, and the optimization becomes stable. Finally, the objective is to find optimal features $F_I$ for injected nodes, which minimize the loss in Eq. \ref{eqn:loss_tdgia} for all target nodes:
\beq{
\label{eqn:obj_tdgia}
\arg\min_{F_I}\frac{1}{|\mathcal{T}|}\sum_{v\in \mathcal{T}} \max(r+\ln p_v,0)^2}

\vpara{Smooth Feature Optimization.}
Under GIA settings, there's a constraint on the range of features of the injected nodes. Otherwise, the defenders can easily filter out injected nodes based on abnormal features. In \model, we simply apply \textit{Clamp} function during optimization process to limit the range of features
\beq{
\label{eqn:clamp}
Clamp(x,min,max)=
\begin{cases}
min, & x<min \\ 
x, & min<x<max \\ 
max, & x>max
\end{cases}}
However, this function may lead to zero gradient. If a feature exceeds the range, it will be stuck at maximal or minimal. To smooth the optimization process of \model, we design a $Smoothmap$ function that remaps features onto $(min,max)$ smoothly by using
\beq{
\label{eqn:smoothmap}
Smoothmap(x,min,max)=\frac{max+min}{2}+\frac{max-min}{2} sin(x).}

\subsection{Overall attack process of \model}
In addition to topological defective edge selection and smooth adversarial optimization, we also include the sequential attack and the use of surrogate models in \model. 

\textit{\textbf{Sequential Attack.}} 
In \model, we adopt the idea of sequential attack~\cite{wang2020scalable} and inject nodes in batches. 
In each batch we add a small number of nodes to the graph, select their edges, and optimize their features. 
We repeat this process until the injection budget is fulfilled. 

\textit{\textbf{Surrogate Model.}} 
Under the black-box setting, the attacker has no information about the models being attacked, thus the attack has to be performed on a surrogate model. 
Specifically, we first train a surrogate model $\mathcal{M}$ using the given training data on the input graph and generate the surrogate labels $\{\hat{y}_v, v\in \mathcal{T}\}$ using $\mathcal{M}$. 
Then we optimize the \model attack to lower the accuracy of $\mathcal{M}$ for $\{\hat{y}_v, v\in \mathcal{T}\}$.
Note that when selecting defective edges, besides $\lambda_v$, we also use the correct probability $p_v$ based on the softmax output of $\mathcal{M}$ on node $v$ for its surrogate label $\hat{y}_v$. We then define the defective score $\mu_v$ as shown in Algorithm \ref{algo:bb}.

\textit{\textbf{Complexity.}}
Given a base model $\mathcal{M}$ with complexity $T$. Usually for GNNs $T=O(ED)$, where $E$ is the number of edges and $D$ being the dimension of input features. 
For edge selection, we needs to inference $\mathcal{M}$ once to generate $p_v$, which costs $O(T)$, and computation for $\lambda_v$ costs $O(E)=o(T)$, so the computation costs $O(T)$. 
For optimization, suppose $\Delta_S$ is the number of epochs and $B$ is the number of batches for sequential injection, the optimization costs $O(\Delta_SBT)$. So the overall complexity for \model is $O(\Delta_SBT)$.
In practice, $\Delta_SB$ for \model is usually set to be smaller than the number of epochs for training $\mathcal{M}$, therefore generating attacks using \model costs less time than training $\mathcal{M}$. \model is very scalable and can work for any GNN as base model.

In summary, the attack of \model is to first inject new nodes (and edges) into the original graph and then learn the features for the injected nodes. 
The injection of new nodes is determined by the topological vulnerabilities of the graph and GNNs. 
The features are learned via the smooth adversarial optimization. 
The overall attack process of \model is illustrated in Algorithm \ref{algo:bb}.

\SetKwInOut{Parameter}{Parameter}

\begin{algorithm}[t]
\SetAlgoLined
\KwIn{Original graph $\mathbf{G}=\{\mathbf{A},\mathbf{F}\}$; surrogate model $\mathcal{M}$;set of target nodes $\mathcal{T}$;}
\KwOut{Attacked graph $\mathbf{G}'=(\mathbf{A}',\mathbf{F}')$;}
\Parameter{Budget on number of injected nodes $b$; budget on degree of each injected node $d$; constraint on range of features $\Delta_F$;}
\tcc{Initialization}
$\mathbf{G}'\leftarrow \mathbf{G}$;$\mathbf{V}_I\leftarrow \mathbf{0}^{N\times N_I}$;$\mathbf{A}_I\leftarrow \mathbf{0}^{N_I\times N_I}$;$\mathbf{F}_I\leftarrow \mathcal{N}(0, \sigma)^{N_I\times D}$\;
\tcc{Sequential injection}
\While{$b>0$}{
  \tcc{Topological Defective Edge Selection}
  \For{$v\in \mathcal{T}$}{
  Calculate the correct probability $p_v$ using $\mathcal{M}(\mathbf{G}')$\;
  Calculate the defective factor $\lambda_v$ using Eq. \ref{eqn:w_tdgia}\;
  Calculate the defective score $\mu_v=(\alpha p_v+(1-\alpha))\lambda_v$\;
  }
  Set up the number of injected nodes $b_{\text{seq}}\leq b$\;
  $\mathbf{V}_I\leftarrow$ Connect $b_{\text{seq}}$ injected nodes to $b_{\text{seq}}\times d$ target nodes\ in $\mathcal{T}$ with the highest defective score $\mu_v$\;
  \tcc{Smooth Adversarial Optimization}
  $\mathbf{F}_I\leftarrow$ Optimize the features of injected nodes smoothly (Eq. \ref{eqn:obj_tdgia}) using $Clamp$ (Eq. \ref{eqn:clamp}) and $Smoothmap$ (Eq. \ref{eqn:smoothmap}) \;
  $b\leftarrow b-b_{\text{seq}}$, update the budget\;
  $\mathbf{A}',\mathbf{F}'\leftarrow$ Update $\mathbf{A}'$ and $\mathbf{F}'$ by $\mathbf{V}_I$, $\mathbf{F}_I$ using Eq. \ref{eqn:A'} and \ref{eqn:F'}\;
  $\mathbf{G}'\leftarrow (\mathbf{A}',\mathbf{F}')$\;
}
\Return{$\mathbf{G}'$}
\caption{The process of Topological Defective Graph Injection Attack (\model).}
\label{algo:bb}
\end{algorithm}


%% file: 5.exp_new.tex
\subsection{Basic Settings}
\label{basic_settings}
\vpara{Datasets.} We conduct our experiments on three large-scale public datasets including 
1) KDD-CUP dataset\footnote{\url{https://www.biendata.xyz/competition/kddcup_2020_formal/}}, a large-scale citation dataset used in KDD-CUP 2020 \textit{Graph Adversarial Attack \& Defense} competition 
2) ogbn-arxiv~\cite{hu2020open}, a benchmark citation dataset
and 
3) Reddit~\cite{hamilton2017inductive}, a well-known online forum post dataset\footnote{A previously-existing dataset originally extracted and  obtained by a third party, and hosted by pushshift.io,  and downloaded from \url{http://snap.stanford.edu/graphsage/##datasets}}. 
Statistics of these datasets and injection constraints are displayed in Table \ref{tab:datasets}. 

\begin{table*}[t]
\small
  \centering
  \caption{\label{tab:match_stats} Statistics of datasets. We consider only unique undirected edges.}
  \begin{tabular}{|c|r|r|r|r|r|c|c|c|c|c|c|}
    \hline \hline
    \multirow{2}*{Dataset} &\multirow{2}*{Nodes} & \multirow{2}*{\makecell{Train\\ nodes}} & \multirow{2}*{\makecell{Val \\nodes}} & \multirow{2}*{\makecell{Test\\ nodes}} & \multirow{2}*{Edges} & \multirow{2}*{Features} & \multirow{2}*{Classes} & \multirow{2}*{\makecell{Feature \\range}} & \multirow{2}*{\makecell{Injection\\feature range}} & \multirow{2}*{\makecell{Injected \\ nodes}} & \multirow{2}*{\makecell{Injection \\ degree limit}}\\
    &&&&&&&&&&&\\
    \hline
    KDD-CUP & 659,574 & 580,000 & 29,574 & 50,000 & 2,878,577 & 100 & 18 & -1.74$\sim$1.63 & -1$\sim$ 1 & 500 &100\\
    ogbn-arxiv & 169,343 & 90,941 & 29,799 & 48,603 & 1,157,799 &128 & 40 & -1.39$\sim$1.64 & -1$\sim$1 & 500 &100\\
    Reddit & 232,965 & 153,932 & 23,699 & 55,334 & 11,606,919 & 602 & 41 & -0.27$\sim$0.26 & -0.25$\sim$0.25 & 500 & 100\\
    \hline \hline
  \end{tabular}
  \label{tab:datasets}
\end{table*}

\vpara{Constraints.} For each dataset, we set up the budget on the number of injected nodes $b=500$ and the budget on degree $d=100$. The feature limit $\Delta_F$ is set according to the range of features in the dataset (Cf Table \ref{tab:datasets}). For experiments on KDD-CUP dataset, most submitted defense methods include preprocessing that filters out nodes with degree approaches to $100$. Therefore, in our experiments, we apply an artificial limit of 88 to avoid being filtered out. These constraints is applied to both \model and baseline attack methods. 

\vpara{Evaluation Metric.} To better evaluate GIA methods, we consider both the performance reduction and the transferability. Our evaluation is mainly based on the weighted average accuracy proposed in KDD-CUP dataset. The metric attaches a weight to each defense model based on its robustness under GIA, i.e. more robust defense gets higher weight. This encourages the adversary to focus on transferability across all defense models, and to design more general attacks. In addition to weighted average accuracy, we also provide the average accuracy among all defense models, and the average accuracy of the Top-3 defense models. 
The three evaluation metrics are formulated below:
\beq{
s_{\text{avg}}=\frac{1}{n}\sum_{i=1}^{n}s_i,
}
\beq{
s_{\text{top-3}}=\frac{1}{n}\sum_{i=1}^{3}s_i,
}
\beq{
\mathbf{f}_t=\sum_{u\in\mathcal{A}'_t(v)} w_v \mathbf{h}_u^{k-1}
}
\beq{
\label{eqn:w}
s_{\text{weighted}}=\sum_{i=1}^{n}w_is_i, \sum_{i=1}^{n}w_i=1, w_1\geq w_2\geq ... \geq w_n.
}
where $s_1,s_2,...s_n$ are descending accuracy scores of $n$ different defense models against one GIA attack, i.e. $s_1\geq s_2...\geq s_n$. 
For KDD-CUP dataset, we use the given weights, for ogbn-arxiv and Reddit, we set the weights in a similar way. More reproducibility details are introduced in Appendix \ref{app:repro}.

\subsection{Attack \& Defense Settings}
\label{sec:exp_settings}

\vpara{Baseline Attack Methods.} We compare our \model approach with different baselines, including 
FGSM~\cite{szegedy2013intriguing}, 
AFGSM~\cite{wang2020scalable}, 
and the SPEIT method~\cite{zheng2020kdd}, the open-source attack method released by the champion team of KDD-CUP 2020. For KDD-CUP dataset we also include the top five attack submissions in addition to the above baselines.
Specifically, FGSM and AFGSM are adapted to the GIA settings with black-box and evasion attacks.
For FGSM~\cite{szegedy2013intriguing}, we randomly connect injected nodes to the target nodes, and optimize their features with inverse KL divergence (Eq. \ref {eqn:loss}). AFGSM~\cite{wang2020scalable} offers an improvement to FGSM, we also adapt it to our GIA settings.
Note that NIPA~\cite{sun2020adversarial} covered in Table~\ref{tab:diffs} is not scalable enough for the large-scale datasets.

\vpara{Surrogate Attack Model.}  GCN~\cite{kipf2016semi} is the most fundamental and most widely-used model among all GNN variants. Vanilla GCNs are easy to attack~\cite{zugner2018adversarial,wang2020scalable}. However, when incorporated with LayerNorm~\cite{ba2016layer}, it becomes much more robust. Therefore, it is used by some top-competitors in KDD-CUP and achieved good defense results. In our experiments, we mainly use GCN as the surrogate model to conduct transfer attacks. We use GCNs (with LayerNorm) with 3 hidden layers of dimension 256, 128, 64 respectively. Following the black-box setting, we first train the surrogate GCN model, perform \model and various GIA on it, and transfer the injected nodes to all defense methods. 

\vpara{Baseline Defense Models.} For KDD-CUP dataset,  it offers 12 best defense submissions (including models and weights), which are considered as defense models. Note that these defense methods are well-formed, which are much more robust than weak methods like raw GCN (without LayerNorm or any other defense mechanism) evaluated in previous works~\cite{zugner2018adversarial, dai2018adversarial, zugner2019adversarial, wang2020scalable, sun2020adversarial}. Most of top attack methods can lower the performance of raw GCN from $68.37\%$ to less than $35\%$. However, they can hardly reduce the $68.57\%$ weighted average accuracy on these defenses by more than $4\%$. 

For Reddit and ogbn-arxiv datasets, we implement the 7 most representative defense GNN models (also appeared in top KDD-CUP defense submissions), GCN~\cite{kipf2016semi} (with LayerNorm), SGCN~\cite{wu2019simplifying}, TAGCN~\cite{du2017topology}, GraphSAGE~\cite{hamilton2017inductive}, RobustGCN~\cite{zhu2019robust}, GIN~\cite{xu2018powerful} and ~\cite{klicpera2018predict} as defense models. We train these models on the original graph and fix them for defense evaluation against GIA methods.
Details of these models are listed in Appendix \ref{app:models}.

\subsection{Performance of \model}

We use GCN(with LayerNorm) as our surrogate attack model for our experiments.
We first evaluate the proposed \model on KDD-CUP dataset. Table \ref{tab:kddcup} illustrates the average performance of \model and other GIA methods over 12 best defense submissions at KDD-CUP competition. Different from previous works, \model aims at the common topological vulnerability of GNN layers, which makes it more transferable cross different defense GNN models. As can be seen, \model significantly outperforms all baseline attack methods by a large margin with more than $8\%$ reduction on weighted average accuracy.

We also test the generalization ability of \model on other datasets. As shown in Table \ref{tab:reddit}, when attacking the 7 representative defense GNN models on Reddit and ogbn-arxiv, \model still shows dominant performance on reducing the weighted average accuracy. This suggests that \model can well generalize across different datasets.  

To summarize, the experiments demonstrate that \model is an effective injection attack method with promising transferability as well as generalization ability. 


\begin{table}[ht]
\small
    \centering
    \caption{ Performance(\%) of different GIA methods on KDD-CUP over 12 best KDD-CUP defense submissions. }
    
    \label{tab:kddcup}
    \begin{tabular}{|c|c|c|c|c|c|c|c|c|}
        \hline \hline
       & \multirow{2}*{\makecell {Attack\\ Method}} & \multirow{2}*{\makecell {Average \\Accuracy}} & \multirow{2}*{\makecell{Top-3\\ Defense}} & \multirow{2}*{\makecell{Weighted\\Average}} &  \multirow{2}*{Reduction} \\
       &&&&&\\
       \cline{1-6}
      Clean & - & 65.54 & 70.02 & 68.57 & -\\
       \hline
      \multirow{5}*{\makecell{ KDD-CUP\\Top-5 \\ Attack \\ Submissions}}  & advers& 63.44 &68.85 &67.09 & 1.48\\
       & dafts & 63.91 &68.50 & 67.02 & 1.55 \\
       &ntt &60.21 & 68.80 & 66.27 & 2.30 \\
&simong   & 60.02     &    68.59  &       66.29 & 2.28\\
&u1234 &     61.18     &          67.95 &  64.87 & 3.70\\
       \hline
        \multirow{3}*{\makecell{Baseline \\Methods}}&FGSM & 59.80 & 67.44 & 65.04 & 3.53 \\
      & AFGSM & 59.22 & 67.37 & 64.74 & 3.83\\
       & SPEIT & 61.89 & 68.16 & 66.13 & 2.44 \\
       \hline
       \model &\model & \textbf{55.00} & \textbf{64.49} & \textbf{60.49} & \textbf{8.08} \\
       
        \hline \hline
    \end{tabular}
\end{table}

\begin{table}[ht]
\small
    \centering
    \caption{ Performance(\%) of different GIA methods on Reddit and ogbn-arxiv over 7 representative defense models. }
    \label{tab:reddit}
    \begin{tabular}{|c|c|c|c|c|c|c|c|c|}
        \hline \hline
       \multirow{2}*{Dataset}& \multirow{2}*{\makecell {Attack\\ Method}} & \multirow{2}*{\makecell {Average \\Accuracy}} & \multirow{2}*{\makecell{Top-3\\ Defense}} & \multirow{2}*{\makecell{Weighted\\Average}} &  \multirow{2}*{Reduction} \\
       &&&&&\\
       \cline{1-6}
      \multirow{5}*{Reddit} & Clean & 94.86 & 95.94 & 95.62 & -\\
       \cline{2-6}
     
       \cline{2-6}
        &FGSM& 92.26 & 94.61 & 93.80 & 1.82\\
       & AFGSM & 91.46 & 94.64 & 93.61 & 2.01\\
       & SPEIT & 93.35 & 94.27 & 93.99 & 1.63\\ 
       
       \cline{2-6}
       
       &\model & \textbf{86.11} & \textbf{88.95} & \textbf{88.14} & \textbf{7.48} \\
       \hline
       \multirow{5}*{ogbn-arxiv} & Clean & 70.86 & 71.61 & 71.34 &  -\\
       \cline{2-6}
     
       \cline{2-6}
       &FGSM&66.40 & 69.57 & 68.62 & 2.72 \\
       & AFGSM & 62.60 & 69.08 & 66.96 & 4.38\\
       & SPEIT & 66.93 & 69.56 & 68.63 & 2.71 \\ 
     
       \cline{2-6}
       &\model & \textbf{57.00} & \textbf{59.23} & \textbf{58.53} & \textbf{12.81}  \\
        \hline \hline
    \end{tabular}
   
\end{table}

\subsection{Ablation Studies}

In this section, we analyse in details the performance of \model under different conditions, and \model's transferability when we use different surrogate models to attack different defense models.

\vpara{Topological Defective Edge Selection.}
In Section \ref{sec:defective} we propose a new edge selection method based on topological properties. We analysis the effect of this method by illustrating experimental results of different edge selection methods in Figure \ref{fig:edge} (a). "Uniform" method connects injected nodes to targeted nodes uniformly, i.e. each target node receives the same number of links from injected nodes, which is the most common strategy used by KDD-CUP candidates. "Random" method randomly assigns links between target nodes and injected nodes. As illustrated, the topological defective edge selection contributes a lot to attack performance, and almost doubles the reduction on weighted accuracy of defense models.

\begin{figure}
   
    \centering
    \mbox
    {
    \begin{subfigure}[Different Edge Selection Methods]{
        \centering
         \includegraphics[width=0.5\linewidth]{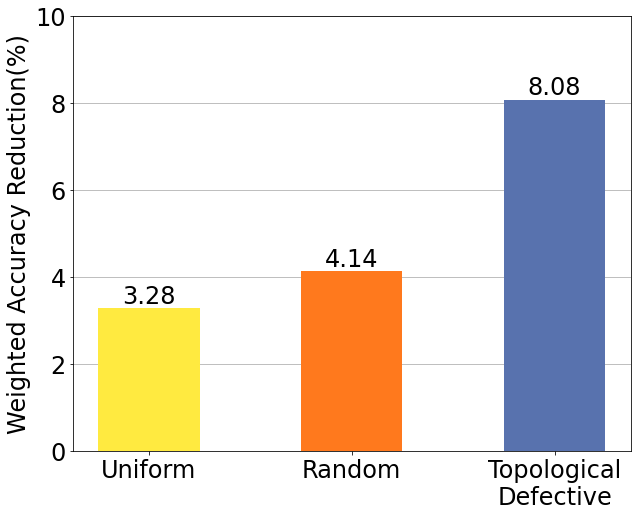}   
         }
    \end{subfigure}
    \hspace{-0.1in}
    \begin{subfigure}[Different Optimization Methods]{
        \centering
         \includegraphics[width=0.5\linewidth]{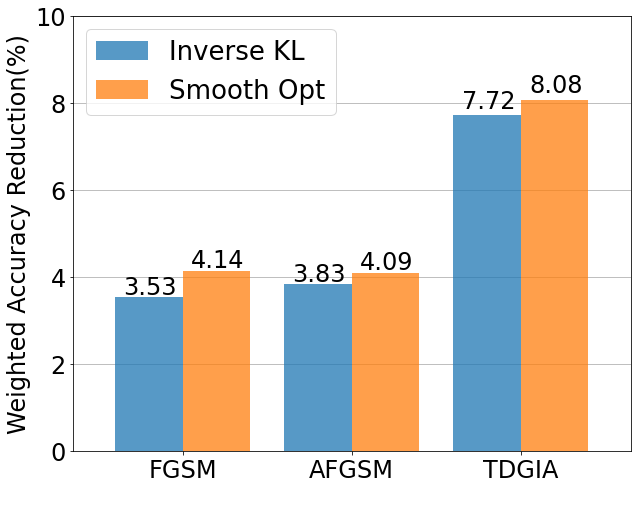}   
         }
    \end{subfigure}
    \hspace{-0.1in}
    }
    \caption{Left:Performance of different edge selection methods based on smooth adversarial optimization. Right: Comparison of smooth adversarial optimization and inverse KL-divergence minimization. Results on KDD-CUP dataset.}
    \label{fig:edge}
    \vspace{-0.15in}
\end{figure}

\vpara{Smooth Adversarial Optimization.}
The smooth adversarial optimization, proposed in Section \ref{sec:smooth}, also has its own advantages. Figure \ref{fig:edge} (b) shows the results of \model and FGSM/AFGSM with/without smooth adversarial optimization. The strategy prevents the issues of gradient explosion and vanishing and does contribute to the attack performance of all three methods. 

\begin{figure}
    \centering
    \includegraphics[width=0.46\textwidth]{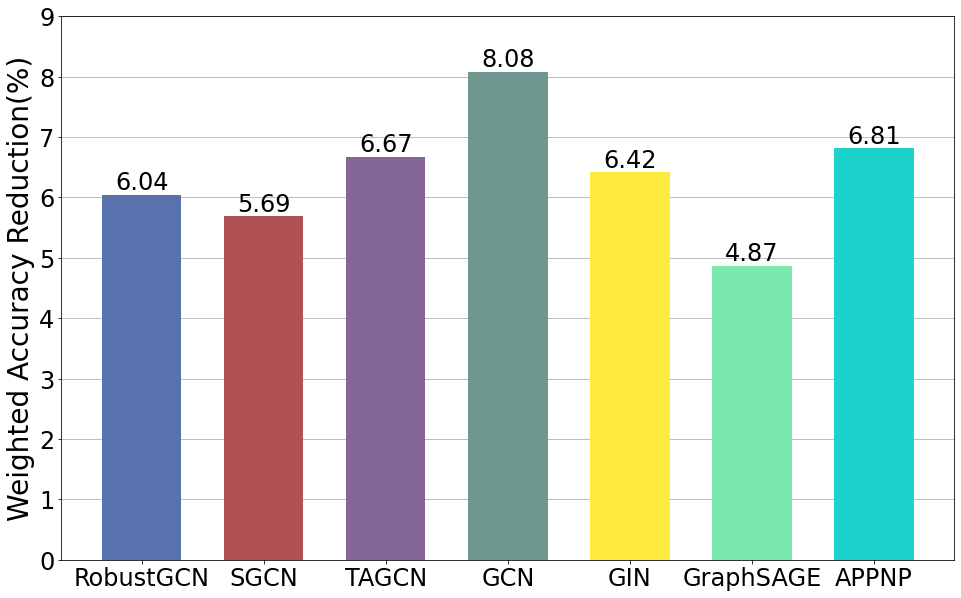}
    \caption{Weighted Accuracy Reduction of \model on KDD-CUP dataset using different surrogate models for attack. GCN yields the best result. }
    \vspace{-0.1in}
    \label{fig:basemodel}
\end{figure}

\vpara{Transferability across Different Models.}
We study the influence of different surrogate models on \model. Figure \ref{fig:basemodel} illustrates the transferability of \model across different models. An interesting result is that GCN turns out to be the best surrogate model, i.e. \model applied on GCN can be better transferred to other models. 
\begin{figure}
    \centering
	\includegraphics[width=0.48\textwidth]{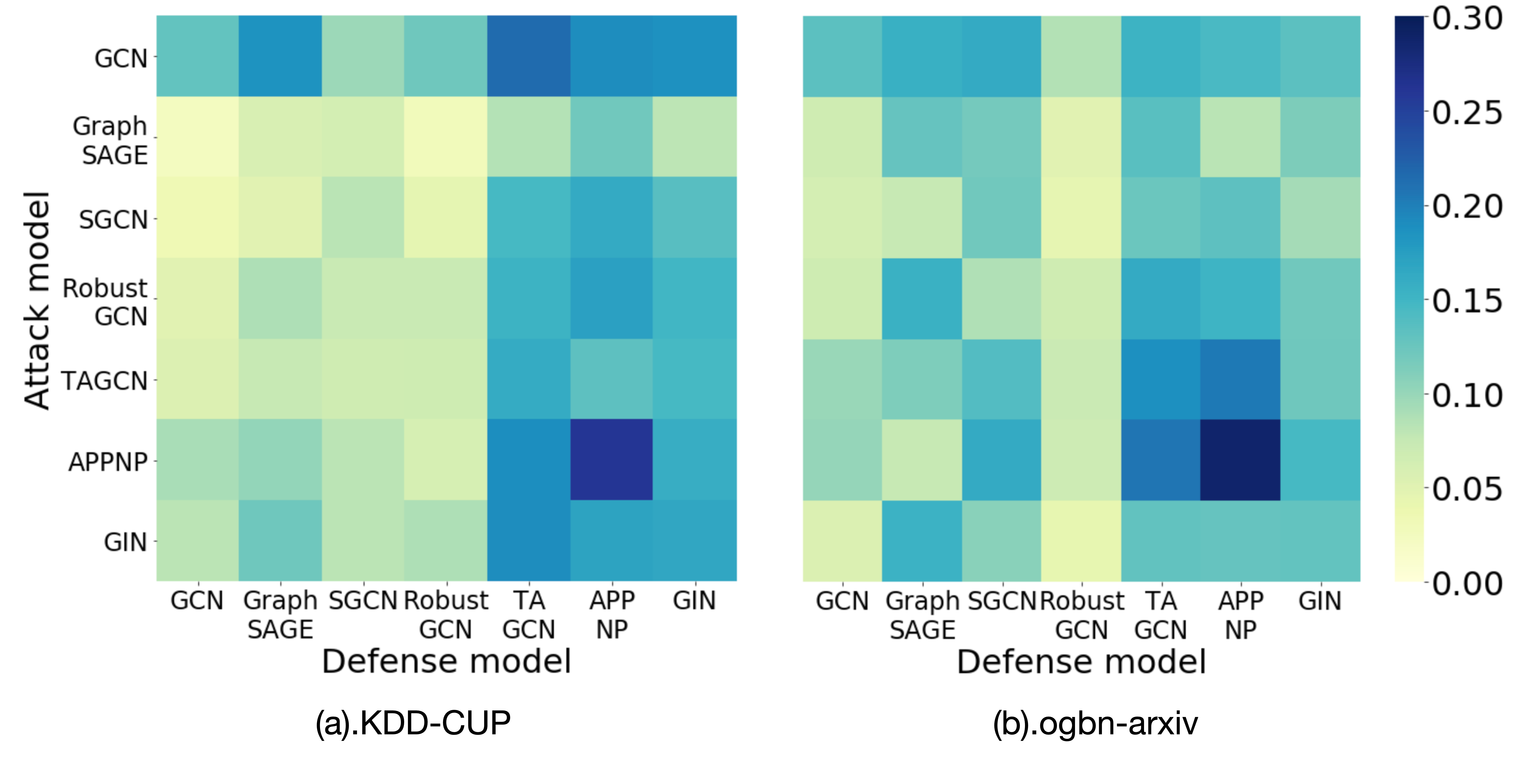}
    \caption{Transferability of \model across different models. Using surrogate models for attack, and evaluate on defense models. Darker color suggests larger performance reduction. Left: KDD-CUP. Right: ogbn-arxiv. }
    \label{fig:conf}
    \vspace{-0.15in}
\end{figure}

Figure \ref{fig:conf} further offers the visualization of transferability of \model on KDD-CUP and ogbn-arxiv. The heat-map shows that \model is effective for whatever surrogate model we use, and can be transferred to all defense GNN models listed in this paper, despite that the scale of transferability may vary. Again, we can see GCN yields attacks with better transferability. A probable explanation may be that most of GNN variants are designed based on GCN, making them more similar to GCN. Therefore, \model can be better transferred using GCN as surrogate models. We also notice that RobustGCN is more robust as defense models, as it is intentionally designed to resist adversarial attacks. Still, our \model is able to reduce its performance. 

\begin{figure}
    
    \centering
    
	\mbox
	{
		\hspace{-0.1in}
		\begin{subfigure}[TDGIA vs. FGSM (ogbn-arxiv)]{
				\centering
				\includegraphics[width = 0.5 \linewidth]{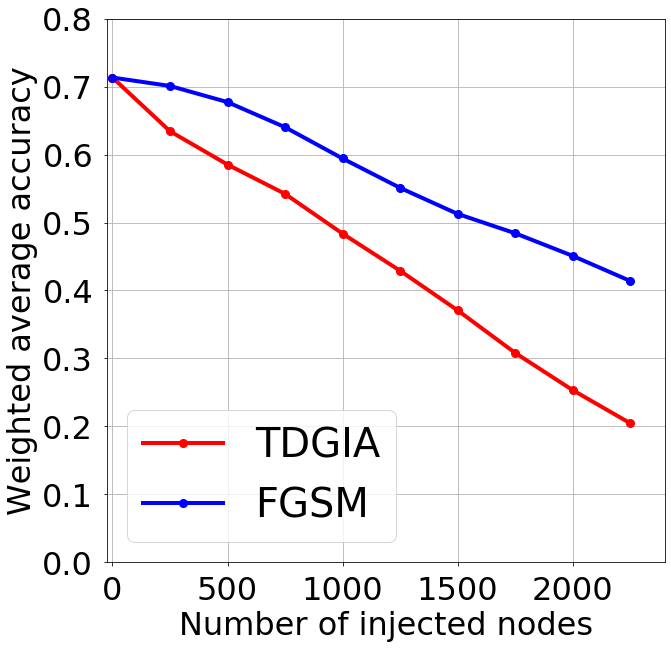}
			}
		\end{subfigure}
        \hspace{-0.1in}
		\begin{subfigure}[TDGIA vs. FGSM (KDD-CUP)]{
				\centering
				\includegraphics[width = 0.5 \linewidth]{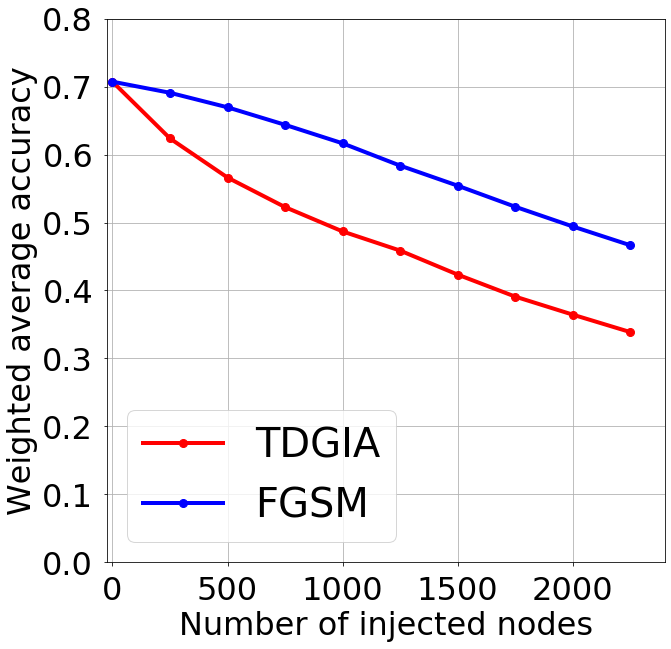}
			}
		\end{subfigure}
	}
	\mbox
	{
		\hspace{-0.1in}
		\begin{subfigure}[TDGIA on Defense Models (ogbn-arxiv)]{
				\centering
				\includegraphics[width = 0.5 \linewidth]{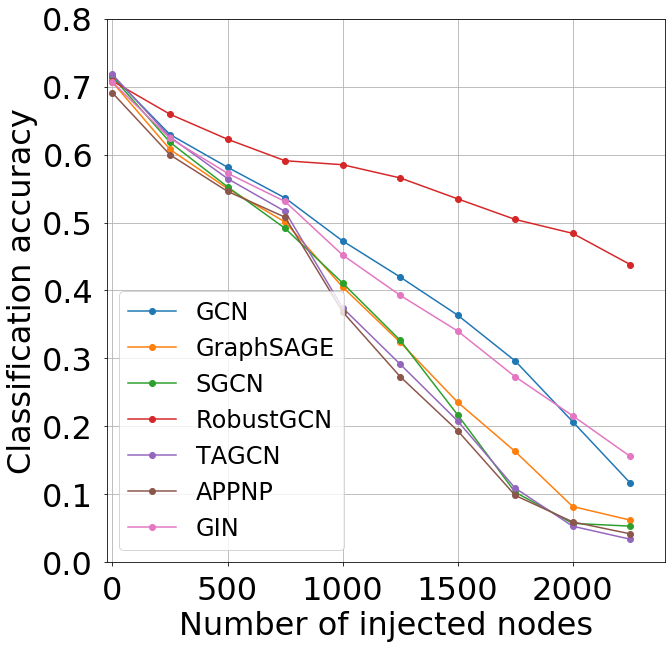}
			}
		\end{subfigure}
		\hspace{-0.1in}
		\begin{subfigure}[TDGIA on Defense Models (KDD-CUP)]{
				\centering
				\includegraphics[width = 0.5 \linewidth]{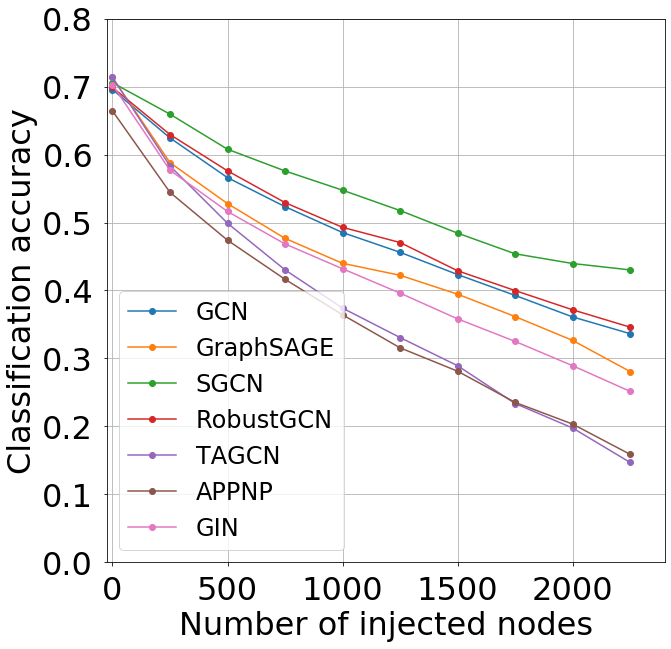}
			}
		\end{subfigure}
		\vspace{-0.1in}
	}
    \caption{\model under different numbers of injected nodes.}
    \vspace{-0.2in}
    \label{fig:diff_injection}
\end{figure}
\vpara{Magnitude of Injection.} In previous experiments, we limit the number of injected nodes under 500. Note that this number is only $1\%$ compared to the number of target nodes. To show the power of \model, we investigate the effect of the magnitude of injection. Figure \ref{fig:diff_injection} (a) (b) show the attack performance of \model and FGSM on ogbn-arxiv and KDD-CUP datasets. For any magnitude of injection, \model always outperforms FGSM significantly, with a gap for more than $10\%$ reduction on weighted average accuracy. 

Figure \ref{fig:diff_injection} (c) (d) further illustrate the detailed performance of \model on different defense models. When we expand the number of injected nodes, there is a continuous performance drop for all defense models. As the number increased to 2000, the accuracy of several defense models (e.g. GraphSAGE, SGCN, APPNP) even drop less than $10\%$. However, this number of injected nodes is still less than $5\%$ the size of target nodes, or $1.4\%$ (ogbn-arxiv) / $0.4\%$ (KDD-CUP) the size of the whole graph. The results demonstrate that most GNN models are quite vulnerable towards \model. 


%% file: 7.appendix.tex
\section{Appendix}
In appendix we follow the citation index in the main article.
\subsection{Reproducibility Details}
\label{app:repro}

In this section, we introduce the experimental details for \model.

\vpara{Surrogate Attack Models.}
Under the black-box setting, we need to train surrogate models for attack. Each model is trained twice using different random seeds. The first one is the surrogate model. The second one is used as defense models. For KDD-CUP, we evaluate directly based on candidate submitted models and parameters. All models are trained for 10000 epochs using Adam, with a learning rate of 0.001 and dropout rate of 0.1. We evaluate the model on the validation set every 20 epochs and select the one with the highest validation accuracy to be the final model.

\vpara{Detailed Description of Baseline Defense Models.}
\label{app:models}
In section \ref{sec:exp_settings}, We only provide detailed introduction to GCN (with LayerNorm). Here we explain in detail about the other defense models. 

\textit{SGCN~\cite{wu2019simplifying}.}
SGCN aims to simplify the GCN structure by removing the activations while improving the aggregation process. The method introduces more aggregation than other methods and thus has a much larger sensitivity area for each node, making the model more robust against tiny local neighborhood perturbations. We use an initial linear transformation that transforms the input into 140 dimensions, and 2 SGC layers with 120 and 100 channels respectively, with $k=4$ and LayerNorm, a final linear transformation that transfers that into the number of classes.

\textit{TAGCN~\cite{du2017topology}.}
TAGCN is a GCN variant that combines multiple-level neighborhoods in every single-layer of GCN. This is the method used by SPEIT, the champion of the competition. The model in our experiments has 3 hidden layers, each with 128 channels and with the propagation factor $k=3$.

\textit{GraphSAGE~\cite{hamilton2017inductive}.}
GraphSAGE represents a type of node-based neighborhood aggregation mechanism, which aggregates direct neighbors on each layer. The aggregation function is free to adjust. Many teams use this framework in their submissions, their submissions vary in aggregation functions. We select a representative GraphSAGE method that aggregates neighborhoods based on $L_2$-norm of neighborhood features. The model has 4 hidden layers, each with 70 dimensions. 

\textit{RobustGCN~\cite{zhu2019robust}.}
RobustGCN is a GCN variant designed to counter adversarial attacks on graphs. The model borrows the idea of random perturbation of features from VAE, and tries to encode both the mean and variation of the node representation and keep being robust against small perturbations. We use 3 hidden layers with 150 dimensions each.

\textit{GIN~\cite{xu2018powerful}.}
GIN is introduced to maximize representation power of GNNs by aggregating self-connected features and neighborhood features of each node with different weights. Team "Ntt Docomo Labs" uses this method as defensive model. We use 4 hidden layers with 144 dimensions each.

\textit{APPNP~\cite{klicpera2018predict}.}
APPNP is designed for fast approximation of personalized prediction for graph propagation. Like SGCN, APPNP propagates on graph dozens of times in a different way and therefore more robust to local perturbations. In application we first transform the input with a 2-layer fully-connected network with hidden size 128, then propagate for 10 times. 

\vpara{Attack Parameters.}
\label{app:attack}
We conduct our attacks using batch-based smooth adversarial optimization. $r$ in Eq. \ref{eqn:loss_tdgia} is set to 4. We follow the Algorithm \ref{algo:bb}. For $\lambda_v$ in Eq. \ref{eqn:w_tdgia}, we take $k_1=0.9$ and $k_2=0.1$, for $\alpha$ mentioned in Algorithm \ref{algo:bb} is set to 0.33. We don't follow exactly the AFGSM description of one-by-one injection, as it costs too much time to optimize on large graphs with hundreds of injected nodes, instead we add nodes in batches, each batch contains nodes equal to $20\%$ of the injection budget, and is optimized under smooth adversarial optimization with Adam optimizer with a learning rate of $1$, features are initialized by $N(0,1)$. Each batch of injected nodes is optimized on surrogate models to lower its prediction accuracy on approximate test labels for 2,000 epochs.

\vpara{Evaluation Mertric.}
The $w$ mentioned in Eq. \ref{eqn:w} is set to $[0.24,0.18,0.12,0.1,0,08,0.07,0.06,0.05,0.04,0.03,0.02,0.01]$ in the KDD-CUP dataset, which offers a variety of 12 top candidate defense submissions. For evaluation on ogbn-arxiv and Reddit, $w$ is set to $[0.3,0,24,0.18,0.12,0.08,0.05,0.03]$ for the 7 defense models. In Figure \ref{fig:diff_injection} the KDD-CUP evaluation is based on the same models and weights as ogbn-arxiv.

\vpara{Additional Information on Datasets.}
In the raw data of the Reddit dataset, unlike other dimensions, dimension 0 and 1 are integers ranging from 1 to 22737. We apply a transformation $f(x)=0.025\times \log(x)$ to regularize their range to $[-0.27, 0.26]$ to match the scale of other dimensions. The ogbn-arxiv dataset has 1,166,243 raw links, however 8,444 of them are duplicated and only 1,157,799 are unique bidirectional links. So we take 1,157,799 as the number of links for ogbn-arxiv dataset. 

\subsection{Generalizing Topological Properties of Single-Layer GNN to Multi-Layer GNNs}
\label{app:layer}
In this section, we elaborate the topological properties of GNNs from single-layer to multi-layer. We show in multi-layer GNNs, the perturbation of $\mathbf{h}$ only relies on $\{\Delta w_{u,t}\|t=1,2...\}$, therefore the topological defective edge selection in Section \ref{sec:gnn-topo} can be generalized to multi-layer GNNs. 

We start from Eq. \ref{eq:uu}, for a GNN layer,
\beq{
\label{eqn:uu3}
\Delta \mathbf{f}_t=\sum_{u\in \mathcal{A}'_t(v)}(\Delta w_{u,t}\mathbf{h}_u+w_{u,t}\Delta \mathbf{h}_u)}

Here, $\Delta w_u=w'_u-w_u, \Delta \mathbf{h}_u=\mathbf{h}'_u-\mathbf{h}_u$ if $u\in \mathcal{A}'_t(v)$ and $u\in \mathcal{A}_t(v)$. For $u\in \mathcal{A}'_t(v),u\notin\mathcal{A}_t(v)$, $\Delta w_u=w'_u, \Delta \mathbf{h}_u=\mathbf{h}'_u$. For GIA, $\mathcal{A}_t(v)\subseteq\mathcal{A}'_t(v)$. 
Recall that $\mathbf{p}_{t,k}=\frac{\partial \mathbf{h}_v^k}{\partial \mathbf{f}_t}$. Then

\beq{
\Delta \mathbf{h}_{v}^k=\sum_{t=0}^{n}\mathbf{p}_{tk}\sum_{u\in \mathcal{A}'_t(v)} (\Delta w_u\mathbf{h}_u^{k-1}+w_u\Delta \mathbf{h}_u^{k-1})}

\beq{
\label{eqn:uu4}
\Delta \mathbf{h}_{v}^k=\sum_t\sum_{u\in \mathcal{A}'_t(v)} (\Delta w_{u,t}\mathbf{p}_{t,k}\mathbf{h}_u^{k-1}+w_{u,t}\mathbf{p}_{tk}\Delta \mathbf{h}_u^{k-1})
}

This suggests that for single-layer GNNs the perturbation only relies on $\{\Delta w_{u,t}\|t=1,2...\}$. 

Now let's generalize this result to multi-layer GNNs using induction. Suppose $\Delta \mathbf{h}_v^k$ can be expressed in the following form

\beq{
\label{ini}
\Delta \mathbf{h}_{v}^k=\sum_t\sum_{u\in \mathcal{A}'_t(v)} (\Delta w_{u,t}\sum_{l=1}^{k}\mathbf{p}_{t,l}\mathbf{h}_u^{l-1}+w_{u,t}\mathbf{p}_{t,k}\Delta \mathbf{h}_u^{0})
}

\noindent For induction, suppose it holds for $k=n$,
\beq{
\begin{split}
\Delta \mathbf{h}_v^{n+1}
&=\sum_t\sum_{u\in \mathcal{A}'_t(v)} (\Delta w_{u,t}\mathbf{p}_{t,k+1}\mathbf{h}_u^{k}+w_{u,t}\mathbf{p}_{t,k+1}\Delta \mathbf{h}_u^{k})\\
&=\sum_t\sum_{u\in \mathcal{A}'_t(v)} (\Delta w_{u,t}\mathbf{p}_{t,k+1}\mathbf{h}_u^{k}+\\
&w_{u,t}\mathbf{p}_{t,k+1}\sum_{t'}\sum_{u'\in \mathcal{A}'_{t'}(u)} (\Delta w_{u',t'}\sum_{l=1}^{k-1}\mathbf{p}_{t',l}\mathbf{h}_u^{l-1}+w_{u',t'}\mathbf{p}_{t',k}\Delta \mathbf{h}_u^{0}))\\
&=\sum_t\sum_{u\in \mathcal{A}'_t(v)} (\Delta w_{u,t}\sum_{l=1}^{k}\mathbf{p}'_{t,l}\mathbf{h}_u^{l-1}+w_{u,t}\mathbf{p}'_{t,k+1}\Delta \mathbf{h}_u^{0})
\end{split}
}

\noindent where $\mathbf{p}'$ is a function of the previous $\textbf{p}$. Therefore Eq. \ref{ini} holds for $k=n+1$. Also it obviously holds for $k=1$, by induction we conclude that Eq. \ref{ini} holds for all $k$. Therefore, 

\beq{
\label{deltah}
\Delta \mathbf{h}_{v}^n=\sum_t\sum_{u\in \mathcal{A}'_t(v)} (\Delta w_{u,t}\sum_{l=1}^{n}\mathbf{p}_{t,l}\mathbf{h}_u^{l-1}+w_{u,t}\mathbf{p}_{t,k}\Delta \mathbf{h}_u^{0})
}

And when we have $w_{u,t}=0,\forall u\notin \mathcal{G}$ for GIA, assuming $\Delta w_{u,t}<<1$, the function becomes

\beq{
\label{attackrule}
\Delta \mathbf{h}_{v}^n=\sum_t\sum_{u\in \mathcal{A}'_t(v)} \Delta w_{u,t}\sum_{l=1}^{n}\mathbf{p}_{t,l}\mathbf{h}_u^{l-1}
}

This means the perturbation on multi-layer GNNs also only relies on $\{\Delta w_{u,t}|t=1,2...\}$. Therefore, \model can capture the topological weaknesses of them, which is also demonstrated what our extensive experiments on multi-layer GNNs.




\subsection{Proof of Lemma \ref{lemma:GIA}}
\label{app:proof}
\begin{proof}
Assume model $\mathcal{M}$ is non-structural-igonorant, then there exist $\mathbf{G}_1=(\mathbf{A}_1,\mathbf{F}),\mathbf{G}_2=(\mathbf{A}_2,\mathbf{F})$, $\mathbf{A}_1\neq\mathbf{A}_2$, $\mathcal{M}(\mathbf{G}_1)\neq\mathcal{M}(\mathbf{G}_2)$. Permute a common node of $\mathbf{G}_1$ and $\mathbf{G}_2$ to position 0, then 
	\begin{displaymath}
	A_1 = 
	\left[ \begin{array}{cc} 
	0 & B_1 \\ 
	C_1 & S_1 \\
    \end{array} \right],
    A_2 = 
	\left[ \begin{array}{cc} 
	0 & B_2 \\ 
	C_2 & S_2 \\
    \end{array} \right]
    \end{displaymath}
    \begin{displaymath}
	B_1, B_2\in \mathbb{R}^{1\times (N-1)}, 
    C_1, C_2\in \mathbb{R}^{(N-1)\times 1}, 
    S_1, S_2\in \mathbb{R}^{(N-1)\times (N-1)}
    \end{displaymath}

\noindent Consider a case of GIA in which nodes with the same features are injected to $\mathbf{G}_1$ and $\mathbf{G}_2$ in a different way, i.e. $\mathbf{G}_1^*=(\mathbf{A}_1^*,\mathbf{F}^*)$, $\mathbf{G}_2^*=(\mathbf{A}_2^*,\mathbf{F}^*)$ where
\begin{displaymath} 
A_1^* = 
\left[ \begin{array}{ccc} 
0 & B_1 & B_2 \\ 
C_1 & S_1 & \textbf{0} \\ 
C_2 & \textbf{0} & S_2 
\end{array} \right],
A_2^* = 
\left[ \begin{array}{ccc} 
0 & B_2 & B_1 \\ 
C_2 & S_2 & \textbf{0} \\ 
C_1 & \textbf{0} & S_1 
\end{array} \right]
\end{displaymath}
\begin{displaymath}
\end{displaymath}
Suppose $\mathcal{M}$ is not GIA-attackable, by Definition \ref{def:gia-attackable}
\beq{
\mathcal{M}(\mathbf{G}_1^*)=\mathcal{M}(\mathbf{G_1}), \mathcal{M}(\mathbf{G}_2^*)=\mathcal{M}(\mathbf{G_2})
}

\noindent However, since $\mathcal{M}$ is permutation invariant, i.e. $\mathbf{G}_1^*$ and $\mathbf{G}_2^*$ are the same graph under permutation, then 
\beq{
\mathcal{M}(\mathbf{G_1})=\mathcal{M}(\mathbf{G}_1^*)=\mathcal{M}(\mathbf{G_2^*})=\mathcal{M}(\mathbf{G_2})}
which contradicts to the initial assumption that $\mathcal{M}(\mathbf{G_1})\neq \mathcal{M}(\mathbf{G_2})$, so $\mathcal{M}$ is GIA-attackable.
\end{proof}

\begin{table} 
\small
    \centering
    \caption{ Full Performance(\%) table on KDD-CUP dataset including 28 competition submissions, baselines methods, and result of GIA using all 7 different surrogate attack models, evaluated on 12 top candidate submitted defenses in KDD-CUP 2020. Best results are bolded.}
    \vspace{-0.1in}
    \label{tab:kddcup-full}
    \begin{tabular}{|c|c|c|c|c|c|c|c|c|}
        \hline \hline
       & \multirow{2}*{\makecell{Attack\\ Method}} & \multirow{2}*{\makecell{Average\\ Accuracy}} & \multirow{2}*{\makecell{Top-3 \\Defense}} & \multirow{2}*{\makecell{Weighted \\Average }} &  \multirow{2}*{Reduction} \\
       &&&&&\\
       \cline{1-6}
      Clean & - & 65.54 & 70.02 & 68.57 & -\\
       \hline
      \multirow{28}*{\makecell{KDD-CUP\\Attacks}}  & advers& 63.44 &68.85 &67.09 & 1.48\\
       & dafts & 63.91 &68.50 & 67.02 & 1.55 \\
       & deepb & 61.44 & 69.4 & 67.26 & 1.31\\
       & dminers & 63.76 & 69.39 & 67.48 & 1.09\\
       & fengari & 63.78 & 69.41 & 67.45 & 1.12\\
       & grapho& 63.75& 69.34& 67.44 & 1.13\\
       & msupsu&65.49   &  69.97  &  68.52 & 0.05\\
       & ntt &60.21 & 68.80 & 66.27 & 2.30 \\
       & neutri & 63.62 &69.42 &67.42 & 1.15\\
       & runz& 63.96 & 69.40 & 67.55 & 1.02 \\
       & speit& 61.97 & 69.49 & 67.32 & 1.25\\
       & selina& 64.67& 69.40& 67.79 & 0.78\\
       & tsail & 63.90 & 69.40 & 67.55 & 1.02 \\
       & cccn & 63.11 & 69.26 & 67.28 & 1.29 \\
       & dhorse  &     63.94     &          69.33      &   67.51 & 1.06\\
&kaige  &         63.90   &            69.41    &     67.49 & 1.08\\
&idvl   &         63.57   &            69.42    &     67.39 & 1.18\\
&hhhvjk   &       65.00   &            69.38   &      67.93 & 0.64\\
&fashui  &      63.69    &           69.42  &       67.42 & 1.15\\
&shengz     &     63.99     &          69.40   &      67.55 & 1.02\\
&sc        &      64.48    &           69.11     &    67.41 & 1.16\\
&simong  & 60.02     &    68.59  &       66.29 & 2.28\\
&tofu     &       63.87  &             69.39  &       67.50 & 1.07\\
&yama     &       64.21   &            68.77  &       67.23 & 1.34\\
&yaowen    &    63.94   &            69.33    &     67.50& 1.07\\
&tzpppp   &       65.01       &        69.38    &     67.94 & 0.63\\
&u1234 &     61.18     &          67.95 &  64.87 & 3.70\\
&zhangs    &      63.73      &         69.43   &      67.51 & 1.06\\
       \hline
        \multirow{7}*{\makecell{Baseline \\Methods}}&FGSM & 59.80 & 67.44 & 65.04 & 3.53 \\
       &\multirow{2}*{\makecell{FGSM\\(Smooth)}}& 58.45 & 67.13 & 64.43 & 4.14\\
       &&&&&\\
      & AFGSM & 59.22 & 67.37 & 64.74 & 3.83\\
       & \multirow{2}*{\makecell{AFGSM\\(Smooth)}} & 58.52 & 67.15 & 64.48 & 4.09\\
       &&&&&\\
       & SPEIT & 61.89 & 68.16 & 66.13 & 2.44 \\
       \hline\hline
       \multirow{7}*{\makecell{\model \\ with \\ different\\surrogate\\models}}&RobustGCN &57.24 & 65.83 & 62.53 & 6.04\\
       &sgcn & 58.01 & 66.28 & 62.88 & 5.69\\
       &tagcn & 58.35 & 65.82 & 62.90 & 5.67\\
       &GCN & \textbf{55.00} & \textbf{64.49} & \textbf{60.49} & \textbf{8.08} \\
       &GIN & 56.83 & 65.70 & 62.15 & 6.42 \\
       &GraphSAGE & 59.35 & 66.43 & 63.70 & 4.87 \\
       & appnp & 55.80 & 65.93 & 61.76 & 6.81\\
       
        \hline \hline
    \end{tabular}
    \vspace{-0.15in}
\end{table}

%% file: main.bbl

\begin{thebibliography}{26}


\ifx \showCODEN    \undefined \def \showCODEN     #1{\unskip}     \fi
\ifx \showDOI      \undefined \def \showDOI       #1{#1}\fi
\ifx \showISBNx    \undefined \def \showISBNx     #1{\unskip}     \fi
\ifx \showISBNxiii \undefined \def \showISBNxiii  #1{\unskip}     \fi
\ifx \showISSN     \undefined \def \showISSN      #1{\unskip}     \fi
\ifx \showLCCN     \undefined \def \showLCCN      #1{\unskip}     \fi
\ifx \shownote     \undefined \def \shownote      #1{#1}          \fi
\ifx \showarticletitle \undefined \def \showarticletitle #1{#1}   \fi
\ifx \showURL      \undefined \def \showURL       {\relax}        \fi
\providecommand\bibfield[2]{#2}
\providecommand\bibinfo[2]{#2}
\providecommand\natexlab[1]{#1}
\providecommand\showeprint[2][]{arXiv:#2}

\bibitem[\protect\citeauthoryear{Athalye, Carlini, and Wagner}{Athalye
  et~al\mbox{.}}{2018}]%
        {athalye2018obfuscated}
\bibfield{author}{\bibinfo{person}{Anish Athalye}, \bibinfo{person}{Nicholas
  Carlini}, {and} \bibinfo{person}{David Wagner}.}
  \bibinfo{year}{2018}\natexlab{}.
\newblock \showarticletitle{Obfuscated gradients give a false sense of
  security: Circumventing defenses to adversarial examples}. In
  \bibinfo{booktitle}{\emph{ICML'18}}. PMLR, \bibinfo{pages}{274--283}.
\newblock


\bibitem[\protect\citeauthoryear{Ba, Kiros, and Hinton}{Ba
  et~al\mbox{.}}{2016}]%
        {ba2016layer}
\bibfield{author}{\bibinfo{person}{Jimmy~Lei Ba}, \bibinfo{person}{Jamie~Ryan
  Kiros}, {and} \bibinfo{person}{Geoffrey~E Hinton}.}
  \bibinfo{year}{2016}\natexlab{}.
\newblock \showarticletitle{Layer normalization}.
\newblock \bibinfo{journal}{\emph{arXiv preprint arXiv:1607.06450}}
  (\bibinfo{year}{2016}).
\newblock


\bibitem[\protect\citeauthoryear{Carlini and Wagner}{Carlini and
  Wagner}{2017}]%
        {carlini2017towards}
\bibfield{author}{\bibinfo{person}{Nicholas Carlini} {and}
  \bibinfo{person}{David Wagner}.} \bibinfo{year}{2017}\natexlab{}.
\newblock \showarticletitle{Towards evaluating the robustness of neural
  networks}. In \bibinfo{booktitle}{\emph{2017 IEEE Symposium on Security and
  Privacy}}. IEEE, \bibinfo{pages}{39--57}.
\newblock


\bibitem[\protect\citeauthoryear{Carlini and Wagner}{Carlini and
  Wagner}{2018}]%
        {carlini2018audio}
\bibfield{author}{\bibinfo{person}{Nicholas Carlini} {and}
  \bibinfo{person}{David Wagner}.} \bibinfo{year}{2018}\natexlab{}.
\newblock \showarticletitle{Audio adversarial examples: Targeted attacks on
  speech-to-text}. In \bibinfo{booktitle}{\emph{2018 IEEE Security and Privacy
  Workshops}}. IEEE.
\newblock


\bibitem[\protect\citeauthoryear{Dai, Li, Tian, Huang, Wang, Zhu, and Song}{Dai
  et~al\mbox{.}}{2018}]%
        {dai2018adversarial}
\bibfield{author}{\bibinfo{person}{Hanjun Dai}, \bibinfo{person}{Hui Li},
  \bibinfo{person}{Tian Tian}, \bibinfo{person}{Xin Huang},
  \bibinfo{person}{Lin Wang}, \bibinfo{person}{Jun Zhu}, {and}
  \bibinfo{person}{Le Song}.} \bibinfo{year}{2018}\natexlab{}.
\newblock \showarticletitle{Adversarial Attack on Graph Structured Data}. In
  \bibinfo{booktitle}{\emph{ICML'18}}.
\newblock


\bibitem[\protect\citeauthoryear{Du, Zhang, Wu, Moura, and Kar}{Du
  et~al\mbox{.}}{2017}]%
        {du2017topology}
\bibfield{author}{\bibinfo{person}{Jian Du}, \bibinfo{person}{Shanghang Zhang},
  \bibinfo{person}{Guanhang Wu}, \bibinfo{person}{Jos{\'e}~MF Moura}, {and}
  \bibinfo{person}{Soummya Kar}.} \bibinfo{year}{2017}\natexlab{}.
\newblock \showarticletitle{Topology adaptive graph convolutional networks}.
\newblock \bibinfo{journal}{\emph{arXiv preprint arXiv:1710.10370}}
  (\bibinfo{year}{2017}).
\newblock


\bibitem[\protect\citeauthoryear{Goodfellow, Pouget-Abadie, Mirza, Xu,
  Warde-Farley, Ozair, Courville, and Bengio}{Goodfellow
  et~al\mbox{.}}{2014a}]%
        {goodfellow2014generative}
\bibfield{author}{\bibinfo{person}{Ian Goodfellow}, \bibinfo{person}{Jean
  Pouget-Abadie}, \bibinfo{person}{Mehdi Mirza}, \bibinfo{person}{Bing Xu},
  \bibinfo{person}{David Warde-Farley}, \bibinfo{person}{Sherjil Ozair},
  \bibinfo{person}{Aaron Courville}, {and} \bibinfo{person}{Yoshua Bengio}.}
  \bibinfo{year}{2014}\natexlab{a}.
\newblock \showarticletitle{Generative adversarial nets}. In
  \bibinfo{booktitle}{\emph{NeurIPS'14}}. \bibinfo{pages}{2672--2680}.
\newblock


\bibitem[\protect\citeauthoryear{Goodfellow, Shlens, and Szegedy}{Goodfellow
  et~al\mbox{.}}{2014b}]%
        {goodfellow2014explaining}
\bibfield{author}{\bibinfo{person}{Ian~J Goodfellow}, \bibinfo{person}{Jonathon
  Shlens}, {and} \bibinfo{person}{Christian Szegedy}.}
  \bibinfo{year}{2014}\natexlab{b}.
\newblock \showarticletitle{Explaining and harnessing adversarial examples}.
\newblock \bibinfo{journal}{\emph{arXiv preprint arXiv:1412.6572}}
  (\bibinfo{year}{2014}).
\newblock


\bibitem[\protect\citeauthoryear{Hamilton, Ying, and Leskovec}{Hamilton
  et~al\mbox{.}}{2017}]%
        {hamilton2017inductive}
\bibfield{author}{\bibinfo{person}{Will Hamilton}, \bibinfo{person}{Zhitao
  Ying}, {and} \bibinfo{person}{Jure Leskovec}.}
  \bibinfo{year}{2017}\natexlab{}.
\newblock \showarticletitle{Inductive representation learning on large graphs}.
  In \bibinfo{booktitle}{\emph{NeurIPS'17}}. \bibinfo{pages}{1024--1034}.
\newblock


\bibitem[\protect\citeauthoryear{Hu, Fey, Zitnik, Dong, Ren, Liu, Catasta, and
  Leskovec}{Hu et~al\mbox{.}}{2020}]%
        {hu2020open}
\bibfield{author}{\bibinfo{person}{Weihua Hu}, \bibinfo{person}{Matthias Fey},
  \bibinfo{person}{Marinka Zitnik}, \bibinfo{person}{Yuxiao Dong},
  \bibinfo{person}{Hongyu Ren}, \bibinfo{person}{Bowen Liu},
  \bibinfo{person}{Michele Catasta}, {and} \bibinfo{person}{Jure Leskovec}.}
  \bibinfo{year}{2020}\natexlab{}.
\newblock \showarticletitle{Open graph benchmark: Datasets for machine learning
  on graphs}.
\newblock \bibinfo{journal}{\emph{arXiv preprint arXiv:2005.00687}}
  (\bibinfo{year}{2020}).
\newblock


\bibitem[\protect\citeauthoryear{Huang, Papernot, Goodfellow, Duan, and
  Abbeel}{Huang et~al\mbox{.}}{2017}]%
        {huang2017adversarial}
\bibfield{author}{\bibinfo{person}{Sandy Huang}, \bibinfo{person}{Nicolas
  Papernot}, \bibinfo{person}{Ian Goodfellow}, \bibinfo{person}{Yan Duan},
  {and} \bibinfo{person}{Pieter Abbeel}.} \bibinfo{year}{2017}\natexlab{}.
\newblock \showarticletitle{Adversarial attacks on neural network policies}.
\newblock \bibinfo{journal}{\emph{arXiv preprint arXiv:1702.02284}}
  (\bibinfo{year}{2017}).
\newblock


\bibitem[\protect\citeauthoryear{Jiang, Li, Zhang, Wang, Wang, Yuan, and
  Wei}{Jiang et~al\mbox{.}}{2020}]%
        {jiang2020drug}
\bibfield{author}{\bibinfo{person}{Mingjian Jiang}, \bibinfo{person}{Zhen Li},
  \bibinfo{person}{Shugang Zhang}, \bibinfo{person}{Shuang Wang},
  \bibinfo{person}{Xiaofeng Wang}, \bibinfo{person}{Qing Yuan}, {and}
  \bibinfo{person}{Zhiqiang Wei}.} \bibinfo{year}{2020}\natexlab{}.
\newblock \showarticletitle{Drug--target affinity prediction using graph neural
  network and contact maps}.
\newblock \bibinfo{journal}{\emph{RSC Advances}} \bibinfo{volume}{10},
  \bibinfo{number}{35} (\bibinfo{year}{2020}), \bibinfo{pages}{20701--20712}.
\newblock


\bibitem[\protect\citeauthoryear{Jin, Li, Xu, Wang, Ji, Aggarwal, and Tang}{Jin
  et~al\mbox{.}}{2021}]%
        {jin2021adversarial}
\bibfield{author}{\bibinfo{person}{Wei Jin}, \bibinfo{person}{Yaxing Li},
  \bibinfo{person}{Han Xu}, \bibinfo{person}{Yiqi Wang},
  \bibinfo{person}{Shuiwang Ji}, \bibinfo{person}{Charu Aggarwal}, {and}
  \bibinfo{person}{Jiliang Tang}.} \bibinfo{year}{2021}\natexlab{}.
\newblock \showarticletitle{Adversarial Attacks and Defenses on Graphs}.
\newblock \bibinfo{journal}{\emph{ACM SIGKDD Explorations Newsletter}}
  \bibinfo{volume}{22}, \bibinfo{number}{2} (\bibinfo{year}{2021}),
  \bibinfo{pages}{19--34}.
\newblock


\bibitem[\protect\citeauthoryear{Kipf and Welling}{Kipf and Welling}{2016}]%
        {kipf2016semi}
\bibfield{author}{\bibinfo{person}{Thomas~N Kipf} {and} \bibinfo{person}{Max
  Welling}.} \bibinfo{year}{2016}\natexlab{}.
\newblock \showarticletitle{Semi-supervised classification with graph
  convolutional networks}.
\newblock \bibinfo{journal}{\emph{arXiv preprint arXiv:1609.02907}}
  (\bibinfo{year}{2016}).
\newblock


\bibitem[\protect\citeauthoryear{Klicpera, Bojchevski, and
  G{\"u}nnemann}{Klicpera et~al\mbox{.}}{2018}]%
        {klicpera2018predict}
\bibfield{author}{\bibinfo{person}{Johannes Klicpera},
  \bibinfo{person}{Aleksandar Bojchevski}, {and} \bibinfo{person}{Stephan
  G{\"u}nnemann}.} \bibinfo{year}{2018}\natexlab{}.
\newblock \showarticletitle{Predict then propagate: Graph neural networks meet
  personalized pagerank}.
\newblock \bibinfo{journal}{\emph{arXiv preprint arXiv:1810.05997}}
  (\bibinfo{year}{2018}).
\newblock


\bibitem[\protect\citeauthoryear{Li, Ji, Du, Li, and Wang}{Li
  et~al\mbox{.}}{2019}]%
        {li2019textbugger}
\bibfield{author}{\bibinfo{person}{J Li}, \bibinfo{person}{S Ji},
  \bibinfo{person}{T Du}, \bibinfo{person}{B Li}, {and} \bibinfo{person}{T
  Wang}.} \bibinfo{year}{2019}\natexlab{}.
\newblock \showarticletitle{TextBugger: Generating Adversarial Text Against
  Real-world Applications}. In \bibinfo{booktitle}{\emph{26th Annual Network
  and Distributed System Security Symposium}}.
\newblock


\bibitem[\protect\citeauthoryear{Sun, Wang, Tang, Hsieh, and Honavar}{Sun
  et~al\mbox{.}}{2020}]%
        {sun2020adversarial}
\bibfield{author}{\bibinfo{person}{Yiwei Sun}, \bibinfo{person}{Suhang Wang},
  \bibinfo{person}{Xianfeng Tang}, \bibinfo{person}{Tsung-Yu Hsieh}, {and}
  \bibinfo{person}{Vasant Honavar}.} \bibinfo{year}{2020}\natexlab{}.
\newblock \showarticletitle{Adversarial Attacks on Graph Neural Networks via
  Node Injections: A Hierarchical Reinforcement Learning Approach}. In
  \bibinfo{booktitle}{\emph{WWW'20}}. \bibinfo{pages}{673--683}.
\newblock


\bibitem[\protect\citeauthoryear{Szegedy, Zaremba, Sutskever, Bruna, Erhan,
  Goodfellow, and Fergus}{Szegedy et~al\mbox{.}}{2013}]%
        {szegedy2013intriguing}
\bibfield{author}{\bibinfo{person}{Christian Szegedy},
  \bibinfo{person}{Wojciech Zaremba}, \bibinfo{person}{Ilya Sutskever},
  \bibinfo{person}{Joan Bruna}, \bibinfo{person}{Dumitru Erhan},
  \bibinfo{person}{Ian Goodfellow}, {and} \bibinfo{person}{Rob Fergus}.}
  \bibinfo{year}{2013}\natexlab{}.
\newblock \showarticletitle{Intriguing properties of neural networks}.
\newblock \bibinfo{journal}{\emph{arXiv preprint arXiv:1312.6199}}
  (\bibinfo{year}{2013}).
\newblock


\bibitem[\protect\citeauthoryear{Wang, Luo, Suya, Li, Yang, and Zheng}{Wang
  et~al\mbox{.}}{2020}]%
        {wang2020scalable}
\bibfield{author}{\bibinfo{person}{Jihong Wang}, \bibinfo{person}{Minnan Luo},
  \bibinfo{person}{Fnu Suya}, \bibinfo{person}{Jundong Li},
  \bibinfo{person}{Zijiang Yang}, {and} \bibinfo{person}{Qinghua Zheng}.}
  \bibinfo{year}{2020}\natexlab{}.
\newblock \showarticletitle{Scalable Attack on Graph Data by Injecting Vicious
  Nodes}.
\newblock \bibinfo{journal}{\emph{arXiv preprint arXiv:2004.13825}}
  (\bibinfo{year}{2020}).
\newblock


\bibitem[\protect\citeauthoryear{Wu, Souza, Zhang, Fifty, Yu, and
  Weinberger}{Wu et~al\mbox{.}}{2019}]%
        {wu2019simplifying}
\bibfield{author}{\bibinfo{person}{Felix Wu}, \bibinfo{person}{Amauri Souza},
  \bibinfo{person}{Tianyi Zhang}, \bibinfo{person}{Christopher Fifty},
  \bibinfo{person}{Tao Yu}, {and} \bibinfo{person}{Kilian Weinberger}.}
  \bibinfo{year}{2019}\natexlab{}.
\newblock \showarticletitle{Simplifying graph convolutional networks}. In
  \bibinfo{booktitle}{\emph{ICML'19}}. PMLR, \bibinfo{pages}{6861--6871}.
\newblock


\bibitem[\protect\citeauthoryear{Xu, Hu, Leskovec, and Jegelka}{Xu
  et~al\mbox{.}}{2018}]%
        {xu2018powerful}
\bibfield{author}{\bibinfo{person}{Keyulu Xu}, \bibinfo{person}{Weihua Hu},
  \bibinfo{person}{Jure Leskovec}, {and} \bibinfo{person}{Stefanie Jegelka}.}
  \bibinfo{year}{2018}\natexlab{}.
\newblock \showarticletitle{How Powerful are Graph Neural Networks?}. In
  \bibinfo{booktitle}{\emph{ICLR'18}}.
\newblock


\bibitem[\protect\citeauthoryear{Ying, He, Chen, Eksombatchai, Hamilton, and
  Leskovec}{Ying et~al\mbox{.}}{2018}]%
        {ying2018graph}
\bibfield{author}{\bibinfo{person}{Rex Ying}, \bibinfo{person}{Ruining He},
  \bibinfo{person}{Kaifeng Chen}, \bibinfo{person}{Pong Eksombatchai},
  \bibinfo{person}{William~L Hamilton}, {and} \bibinfo{person}{Jure Leskovec}.}
  \bibinfo{year}{2018}\natexlab{}.
\newblock \showarticletitle{Graph convolutional neural networks for web-scale
  recommender systems}. In \bibinfo{booktitle}{\emph{KDD'18}}.
  \bibinfo{pages}{974--983}.
\newblock


\bibitem[\protect\citeauthoryear{Zheng, Fei, Li, Liu, Hu, and Sun}{Zheng
  et~al\mbox{.}}{2020}]%
        {zheng2020kdd}
\bibfield{author}{\bibinfo{person}{Qinkai Zheng}, \bibinfo{person}{Yixiao Fei},
  \bibinfo{person}{Yanhao Li}, \bibinfo{person}{Qingmin Liu},
  \bibinfo{person}{Minhao Hu}, {and} \bibinfo{person}{Qibo Sun}.}
  \bibinfo{year}{2020}\natexlab{}.
\newblock \bibinfo{booktitle}{\emph{KDD CUP 2020 ML Track 2 Adversarial Attacks
  and Defense on Academic Graph 1st Place Solution}}.
\newblock https://github.com/Stanislas0/KDD\_CUP\_2020\_MLTrack2\_SPEIT.
\newblock


\bibitem[\protect\citeauthoryear{Zhu, Zhang, Cui, and Zhu}{Zhu
  et~al\mbox{.}}{2019}]%
        {zhu2019robust}
\bibfield{author}{\bibinfo{person}{Dingyuan Zhu}, \bibinfo{person}{Ziwei
  Zhang}, \bibinfo{person}{Peng Cui}, {and} \bibinfo{person}{Wenwu Zhu}.}
  \bibinfo{year}{2019}\natexlab{}.
\newblock \showarticletitle{Robust graph convolutional networks against
  adversarial attacks}. In \bibinfo{booktitle}{\emph{KDD'19}}.
  \bibinfo{pages}{1399--1407}.
\newblock


\bibitem[\protect\citeauthoryear{Z{\"u}gner, Akbarnejad, and
  G{\"u}nnemann}{Z{\"u}gner et~al\mbox{.}}{2018}]%
        {zugner2018adversarial}
\bibfield{author}{\bibinfo{person}{Daniel Z{\"u}gner}, \bibinfo{person}{Amir
  Akbarnejad}, {and} \bibinfo{person}{Stephan G{\"u}nnemann}.}
  \bibinfo{year}{2018}\natexlab{}.
\newblock \showarticletitle{Adversarial attacks on neural networks for graph
  data}. In \bibinfo{booktitle}{\emph{KDD'18}}. \bibinfo{pages}{2847--2856}.
\newblock


\bibitem[\protect\citeauthoryear{Z{\"u}gner and G{\"u}nnemann}{Z{\"u}gner and
  G{\"u}nnemann}{2019}]%
        {zugner2019adversarial}
\bibfield{author}{\bibinfo{person}{Daniel Z{\"u}gner} {and}
  \bibinfo{person}{Stephan G{\"u}nnemann}.} \bibinfo{year}{2019}\natexlab{}.
\newblock \showarticletitle{Adversarial attacks on graph neural networks via
  meta learning}.
\newblock \bibinfo{journal}{\emph{arXiv preprint arXiv:1902.08412}}
  (\bibinfo{year}{2019}).
\newblock


\end{thebibliography}
